\documentclass{article}

\usepackage{microtype}
\usepackage{graphicx}
\usepackage{subfigure}
\usepackage{booktabs}

\usepackage{hyperref}

\DeclareMathAlphabet\mathbfcal{OMS}{cmsy}{b}{n}

\newcommand{\mat}[1]{\mathbf{#1}}

\usepackage{etoolbox}
\makeatletter
\preto{\@tabular}{\parskip=0pt}
\makeatother

\usepackage{colortbl}

\newcommand{\band}{\rowcolor{gray!20}}

\usepackage{hyperref}
\usepackage{url}
\usepackage{dirtytalk}

\newcommand{\eref}[1]{Eq.~(\ref{#1})}

\usepackage{adjustbox}
\usepackage{array}
\usepackage{booktabs}

\usepackage{placeins}

\usepackage{titlesec}

\usepackage[accepted]{icml2025}

\usepackage{amsmath}
\usepackage{amssymb}
\usepackage{mathtools}
\usepackage{amsthm}
\usepackage[utf8]{inputenc}
\usepackage{enumitem}

\usepackage{tcolorbox}
\newtcolorbox{codebox}{
    colback=black!5!white, colframe=black!75!white,
    boxrule=0.5mm, arc=3mm, fontupper=\ttfamily
}

\theoremstyle{plain}
\newtheorem{theorem}{Theorem}[section]

\newtheorem{lemma}[theorem]{Lemma}

\theoremstyle{definition}

\theoremstyle{remark}

\usepackage{thmtools,thm-restate}

\usepackage[textsize=tiny]{todonotes}

\icmltitlerunning{Preference Optimization via Contrastive Divergence}

\begin{document}

\twocolumn[
\icmltitle{Preference Optimization via Contrastive Divergence: \\ Your Reward Model is Secretly an NLL Estimator}

\icmlsetsymbol{equal}{*}

\begin{icmlauthorlist}
\icmlauthor{Zhuotong Chen}{yyy}
\icmlauthor{Fang Liu}{yyy}
\icmlauthor{Xuan Zhu}{yyy}
\icmlauthor{Haozhu Wang}{yyy}
\icmlauthor{Yanjun Qi}{yyy}
\icmlauthor{Mohammad Ghavamzadeh}{comp}
\end{icmlauthorlist}

\icmlaffiliation{yyy}{Amazon Web Services, Santa Clara, CA, United States}
\icmlaffiliation{comp}{Amazon Artificial General Intelligence (AGI), Sunnyvale, CA, United States}

\icmlcorrespondingauthor{Zhuotong Chen}{zhuotong@amazon.com}

\icmlkeywords{Preference optimization, Maximum likelihood estimation, Contrastive divergence, Ranking noise contrastive estimation}

\vskip 0.3in
]

\printAffiliationsAndNotice{}  %

\begin{abstract}
Existing studies on preference optimization (PO) have centered on constructing pairwise preference data following simple heuristics, 
such as maximizing the margin between preferred and dispreferred completions based on human (or AI) ranked scores.
However, none of these heuristics has a full theoretical justification. 
In this work, we develop a novel PO framework that provides theoretical guidance to effectively sample dispreferred completions. 
To achieve this, we formulate PO as minimizing the negative log-likelihood (NLL) of a probability model and propose to estimate its normalization constant via a sampling strategy.
As we will demonstrate, these estimative samples can act as dispreferred completions in PO. 
We then select contrastive divergence (CD) as the sampling strategy,
and propose a novel MC-PO algorithm that applies the Monte Carlo (MC) kernel from CD to sample \textit{hard negatives} w.r.t. the parameterized reward model.
Finally, we propose the OnMC-PO algorithm, an extension of MC-PO to the online setting. 
On popular alignment benchmarks, MC-PO outperforms existing SOTA baselines, and OnMC-PO leads to further improvement.
\end{abstract}

\section{Introduction}
\label{sec: introduction}
While large language models (LLMs) learn broad world knowledge, aligning their behavior precisely with human values is challenging due to the unsupervised nature of their training.
Reinforcement learning from human feedback (RLHF) \citep{ouyang2022training} has emerged as a class of effective algorithms to align LLMs \citep{schulman2017proximal}.
Recent works on direct preference optimization (DPO) \citep{rafailov2024direct} and its variants (e.g., Identity preference optimization \citep{azar2024general}) directly optimize an LLM to adhere to human values, without explicit reward modeling or RL.

The data for these algorithms are often collected in the form of preferences \citep{ziegler2019fine}.
This leads to more consistent labeling across human annotators as it reduces their cognitive load and avoids the need for absolute judgment, which can be more subjective. 
Existing studies on PO have predominately considered creating pairwise preference data via simple heuristics,
such as choosing a dispreferred completion by maximizing the gap with the preferred response in terms of human (AI) ratings \citep{tunstall2023zephyr,lambert2024t}.
However,
none of these heuristics has a full theoretical justification.
Here, we ask the question:
{\em ``How to choose dispreferred completion(s) for PO?"}

To answer this question,
we develop a novel PO framework that provides theoretical guidance on developing effective sampling strategies to select dispreferred completions.
To achieve this, we formulate PO as minimizing the negative log-likelihood (NLL) of a probability model. 
Unfortunately, NLL includes a normalization constant that is in the form of an integral, and thus, usually intractable. 
To address this issue, we propose to estimate the normalization constant using a sampling strategy~\citep{naesseth2024elementssequentialmontecarlo}. 
For instance, the ranking noise contrastive estimation applies conditional importance sampling to generate samples from a proposal distribution and uses them to approximate the integral within the normalization constant~\citep{olmin2024connection}. As we will show, these samples can act as dispreferred completions in PO, thus establishing a connection between the proposed NLL estimation and existing PO algorithms. 
Such a connection enables us to apply advanced sampling strategies, from the literature on estimating the normalization constant in NLL, for generating dispreferred completions in PO.

After formulating PO as an NLL minimization problem, we propose to use an advanced sampling strategy, contrastive divergence (CD), to estimate the normalization constant in this NLL. 
CD applies Markov chain Monte Carlo (MCMC) sampling to approximate the gradient of the log-normalization constant \citep{hinton2002training}. 
The central component in CD is the MC kernel that generates the next sample conditioned on the current one. 
The MC kernel produces samples with high probability mass w.r.t. the probability model, which leads to accurate estimation for the gradient of the log-normalization constant. 
In our PO formulation, we define the probability model to be proportional to the log-likelihood of the target policy. Thus, sampling proportionally to the probability model can be interpreted as selecting a \textit{hard negative}, i.e.,~a dispreferred completion that closely resembles the preferred one, and thus, makes it challenging for the model to distinguish between them. 
We demonstrate the effectiveness of this sampling strategy both theoretically and empirically.

\begin{figure}[t]
\centering
\begin{minipage}{.99\columnwidth}
\centering
    \includegraphics[width = \linewidth]{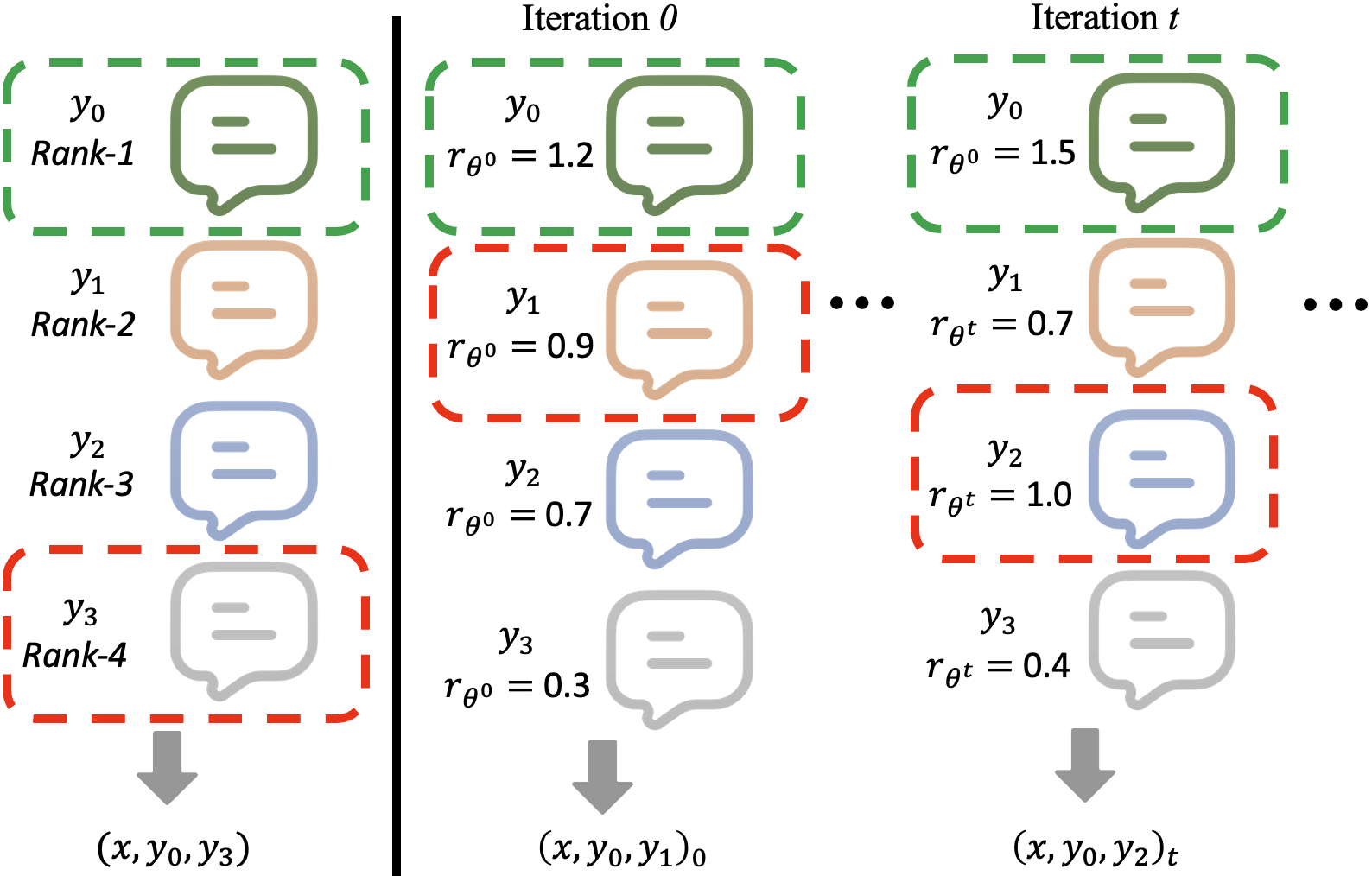}
\end{minipage}
\caption{
\textbf{Left:} existing studies choose a dispreferred completion as the one that maximizes the gap with the preferred completion based on human (or AI) ranked scores.
\textbf{Right:} we propose theoretical guidance to sample dispreferred completion(s) proportionally to the parameterized reward model.
As the parameters evolve during training, the sampling of dispreferred completion changes.
}
\label{fig: high level demonstration}
\end{figure}

We summarize our main contributions below.
\begin{itemize}
    \item 
    We present a novel PO framework that provides theoretical guidance on developing sampling strategies to generate dispreferred completions.
    This is achieved by formulating the alignment problem as NLL estimation and solve it via a sampling strategy.
    \item 
    We propose MC-PO as an offline PO algorithm.
    Given a preference dataset that consists of multiple dispreferred completions for each input prompt,
    MC-PO applies a MC kernel from CD to select \textit{hard negatives} as dispreferred completions proportionally to the log-likelihood of the target policy.
    \item
    Our theoretical analysis suggests that sampling preferred completions from the target policy leads to an unbiased gradient estimation of the normalization constant. 
    Building on this insight, we propose OnMC-PO, an extension of MC-PO to the online settings.
    \item 
    We demonstrate that MC-PO outperforms existing SOTA baselines, and OnMC-PO leads to further improvement.
    Moreover, we numerically validate the effectiveness of various sampling strategies, prove that the MC kernel leads to the optimal model performance.
\end{itemize}

\section{Preference Optimization as NLL Estimation}
\label{sec: po as nll estimation}
In this section,
we revisit the PO objective function (Sec. \ref{sec:background preference optimization}) and formulate it as minimizing the NLL of a probability model (Sec. \ref{sec: preference optimization as nll estimation}).
Then we apply a sampling-based approach to solve this (Sec. \ref{sec: preference optimization via sampling-based solution}).

\subsection{Background: Preference Optimization}
\label{sec:background preference optimization}
RLHF aims to align a target policy $\pi_{\theta}$ with human preference based on an reward model $r(\mat{x}, \mat{y})$.
This optimization problem is formulated as
\begin{equation}
\label{eq: rlhf objective function}
\max \limits_{\pi_{\boldsymbol\theta}}
\mathbb{E}_{\substack{\mat{x} \sim \rho,\\ \mat{y} \sim \pi_{\boldsymbol\theta}(\cdot|x)}}
[r(\mat{x}, \mat{y})]
-
\beta \cdot KL[\pi_{\boldsymbol\theta}(\mat{y} | \mat{x}) || \pi_{\rm ref}(\mat{y} | \mat{x})],
\end{equation}
where $\rho$ represents the distribution over input prompts, 
$\beta$ is a hyper-parameter controlling the deviation from the reference policy $\pi_{\rm ref}$.
The reward model $r(\mat{x}, \mat{y})$ is typically estimated from empirically observed data and cannot accurately represent real-world human preference.
The KL-divergence constraint prevents the model from over-fitting the estimated reward model \citep{skalse2022defining},
as well as maintaining the generation diversity and preventing mode-collapse to single high-reward completions.

Following prior works \citep{go2023aligning},
the closed-form solution to the KL-constrained reward maximization in Eq.~\eqref{eq: rlhf objective function} can be derived as,
\begin{equation}
\label{eq: rlhf optimal solution}
\pi^{\ast}(\mat{y} | \mat{x})
=
\frac{1} {Z(\mat{x})}
\pi_{\rm ref}(\mat{y} | \mat{x}) 
\exp 
\Big(
\frac{1} {\beta} r(\mat{x}, \mat{y})
\Big).
\end{equation}
The partition function $Z(\mat{x})$ ensures that $\pi^{\ast}$ is a valid probability conditioned on any $\mat{x}$.
$Z(\mat{x})$ is typically intractable to compute since the output space is combinatorially large,
which makes this representation hard to utilize in practice.

To estimate $\pi^{\ast}$,
DPO reparameterizes the reward model in terms of the target policy, which enables directly optimizing the target policy
from pairwise preference dataset.
DPO is derived from rearranging the closed-form solution of $\pi^{\ast}$ in Eq.~\eqref{eq: rlhf optimal solution} to express the reward function in terms of its corresponding optimal policy,
then substituting this expression into the Bradley-Terry model \citep{bradley1952rank}.
This yields the DPO objective function,
\begin{equation}
\label{eq:dpo}
\min_{r_{\boldsymbol{\theta}}} 
\underset{(\mat{x}, \mat{y}_0, \mat{y}_1) \sim \mathcal{D}}{\mathbb{E}}
\left[ 
    -\log \sigma\left( 
        \beta r_{\boldsymbol\theta}(\mat{x}, \mat{y}_0) 
        - \beta r_{\boldsymbol\theta}(\mat{x}, \mat{y}_1)
    \right)
\right],
\end{equation}
where $r_{\boldsymbol\theta}(\mat{x}, \mat{y}) := \log \frac{\pi_{\boldsymbol\theta}(\mat{y} | \mat{x})} {\pi_{\rm ref}(\mat{y} | \mat{x})}$ is the parameterized reward model,
$\mat{y}_0$ and $\mat{y}_1$ represent preferred and dispreferred completions, respectively,
$\mathcal{D}$ is the distribution of pairwise preference data.
DPO optimizes the target policy to distinguish between preferred and dispreferred completions, conditioned on the same input prompt.

Existing studies on PO have primarily focused on creating pairwise preference data using heuristic approaches.
In the next, we offer theoretical insights into sampling dispreferred completions by framing PO as NLL estimation.

\subsection{Preference Optimization as NLL Estimation}
\label{sec: preference optimization as nll estimation}
\paragraph{Background: NLL estimation.}
Unnormalized models can be used to approximate complex data distributions.
However,
estimating unnormalized models is not straightforward since the NLL estimation involves the typically intractable normalization constant.
Given some observations from the target distribution,
we seek to approximate it with a parametric probability model,
\begin{equation}
\label{eq: expression for probability distribution}
p_{\boldsymbol\theta}(\mat{y} | \mat{x})
:=
\frac{\tilde{p}_{\boldsymbol\theta}(\mat{y} | \mat{x})}
{Z_{\boldsymbol\theta}(\mat{x})},
\;\;
Z_{\boldsymbol\theta}(\mat{x})
=
\int
\tilde{p}_{\boldsymbol\theta}(\mat{y}' | \mat{x})
d \mat{y}',
\end{equation}
where $\tilde{p}_{\boldsymbol\theta}$ is an unnormalized model and $Z_{\boldsymbol\theta}(\mat{x})$ is its normalization constant.
The NLL estimation minimizes the negative log-likelihood of $p_{\boldsymbol\theta}$ of predicting these observations.
Roughly speaking,
as the number of observations approaches to infinity,
the NLL estimation results in a parametric probability model 
that increasingly approximates the target distribution.

\paragraph{Proposed Formulation: PO as NLL estimation.}
In this work,
we define the following probability model that is closely related to the expression of $\pi^{\ast}$ in Eq.~\eqref{eq: rlhf optimal solution}.
\begin{align}
\label{eq: mle parameterized policy}
& p_{\boldsymbol\theta}(\mat{y} | \mat{x})
: =
\frac{1} {Z_{\boldsymbol\theta}(\mat{x})}
\mu(\mat{y} | \mat{x}) 
\exp\Big(
\beta
r_{\boldsymbol\theta}(\mat{x}, \mat{y})
\Big),
\\ \nonumber
&
Z_{\boldsymbol\theta}(\mat{x})
=
\int
\mu(\mat{y} | \mat{x})
\exp 
\Big(
\beta
r_{\boldsymbol\theta}(\mat{x}, \mat{y})
\Big)
d\mat{y}
,
\\ \nonumber
&
r_{\boldsymbol\theta}(\mat{x}, \mat{y}) = \log \frac{\pi_{\boldsymbol\theta}(\mat{y} | \mat{x})} {\pi_{\rm ref}(\mat{y} | \mat{x})}.
\end{align}
where 
$\mu$ is a proposal distribution that we can sample from.
For any set of parameters $\boldsymbol\theta$,
we assume that $p_{\boldsymbol\theta}$ covers the support of $\pi^{\ast}$, such as $p_{\boldsymbol\theta}(\mat{y} | \mat{x}) > 0$ whenever $\pi^{\ast}(\mat{y} | \mat{x}) > 0$, for all $\mat{x} \sim \rho$. 
In this expression,
the unnormalized model is defined as 
$\tilde{p}_{\boldsymbol\theta}(\mat{y} | \mat{x})
:=
\mu(\mat{y} | \mat{x}) 
\exp\Big(
\beta
r_{\boldsymbol\theta}(\mat{x}, \mat{y})
\Big)$.
To estimate $\pi^{\ast}$ with $p_{\boldsymbol\theta}$,
the NLL estimation minimizes the negative log-likelihood of $p_{\boldsymbol\theta}$ to predict observations sampled from $\pi^{\ast}$,
\begin{align}
\label{eq: nll estimation objective function}
& \boldsymbol\theta^{\ast} = \arg \min \limits_{\boldsymbol\theta}
\mathbb{E}_{\mat{x} \sim \rho, \; \mat{y} \sim \pi^{\ast}(\cdot|x)}
\Big[
\mathcal{L}^{NLL}(\boldsymbol\theta, \mat{x}, \mat{y})
\Big],
\\ \nonumber
& {\rm where}\;
\mathcal{L}^{NLL}(\boldsymbol\theta, \mat{x}, \mat{y})
=
-
\beta
r_{\boldsymbol\theta}(\mat{x}, \mat{y})
+ \log Z_{\boldsymbol\theta}(\mat{x}).
\end{align}
Recall in Eq.~\eqref{eq: mle parameterized policy} that the reward model can be represented in terms of the target policy,
which allows for optimizing the target policy by solving the NLL estimation.

In practice,
the first term of $\mathcal{L}^{NLL}$ in Eq.~\eqref{eq: nll estimation objective function} is typically easy to optimize as the gradient of the target policy (e.g., the LLM) can be computed using existing automatic differentiation software.
However,
optimizing the normalization constant is non-trivial.
In the next,
we focus on sampling-based approaches to estimate the normalization constant.

\subsection{Preference Optimization via Sampling-Based Solution for Its NLL Estimation Formulation}
\label{sec: preference optimization via sampling-based solution}
\paragraph{Proposed: PO via sampling-based solution of NLL estimation.}
Importance sampling applies samples from a proposal distribution to estimate the normalization constant \cite{naesseth2024elementssequentialmontecarlo}.
Ranking noise contrastive estimation (RNCE) \citep{olmin2024connection}, as a more advanced sampling approach, 
utilizes both importance samples and true observations from the target distribution to estimate the intractable term.
Given one observation $\mat{y}_0$ sampled from $\pi^{\ast}$
and $M$ i.i.d. noisy samples from a proposal distribution,
RNCE optimizes to classify as $\mat{y}_0$ coming from the true distribution.

\begin{restatable}{proposition}{rnce}
\label{prop: ranking noise contrastive estimation objective}
Suppose that we have $\mat{y}_0 \sim \pi^{\ast}(\mat{y} | \mat{x})$,
and $M$ noisy samples $\{\mat{y}_i\}_{i=1}^M$,
where each $\mat{y}_i$ is sampled from a proposal distribution, $\mat{y}_i \sim \mu(\mat{y} | \mat{x})$.
Then RNCE approximates the NLL estimation as follows,
\begin{align}
\label{eq: ranking noise contrastive estimation}
& \mathcal{L}^{Sample}(\boldsymbol\theta, \mat{x}, \mat{y}_0)
\\ \nonumber
& =
-
\beta
r_{\boldsymbol\theta}(\mat{x}, \mat{y}_0)
+
\log
{\sum_{i=0}^M \exp
\Big(
\beta
r_{\boldsymbol\theta}(\mat{x}, \mat{y}_i)
\Big)}.
\end{align}
\end{restatable}
The detailed derivation is in Appendix \ref{sec: RNCE objective function derivation}.
Compared to the NLL estimation in Eq.~\eqref{eq: nll estimation objective function},
RNCE approximates the intractable normalization constant as $Z_{\boldsymbol\theta}(\mat{x})
=
\frac{1}{M+1}
\sum_{i=0}^M \exp
\Big(
\beta
r_{\boldsymbol\theta}(\mat{x}, \mat{y}_i)
\Big)$,
with both $\mat{y}_0$ sampled from $\pi^{\ast}$ and noisy samples $\{\mat{y}_i\}_{i=1}^M$ from a proposal distribution (notice that $\frac{1}{M+1}$ is cancelled by taking the gradient of $\log \mathcal{L}^{Sample}$).
Consequently,
$\mathcal{L}^{Sample}$ is equivalent to a cross-entropy loss
that optimizes the model to classify $\mat{y}_0$ as the correct prediction among all ($M+1$) candidates.

\paragraph{Proposed: Existing PO can be formulated as sampling-based solutions of NLL estimation.}
In the sampling-based solution from Eq.~\eqref{eq: ranking noise contrastive estimation},
we consider the true observation $\mat{y}_0$ as the preferred completion,
and noise samples from the proposal distribution as dispreferred completions.
This leads to an expression of PO as follows, with letting $M=1$,
\begin{align*}
& 
-
\beta
r_{\boldsymbol\theta}(\mat{x}, \mat{y}_0)
+
\log
{\sum_{i=0}^1 \exp
\Big(
\beta
r_{\boldsymbol\theta}(\mat{x}, \mat{y}_i)
\Big)}
\\
&
=
- \log
\frac{
\exp
\Big(
\beta
r_{\boldsymbol\theta}(\mat{x}, \mat{y}_0)
\Big)
}
{
\exp
\Big(
\beta
r_{\boldsymbol\theta}(\mat{x}, \mat{y}_0)
\Big)
+
\exp
\Big(
\beta
r_{\boldsymbol\theta}(\mat{x}, \mat{y}_1)
\Big)
},
\\
& =
- \log \sigma
\Big(
\beta
r_{\boldsymbol\theta}(\mat{x}, \mat{y}_0)
-
\beta
r_{\boldsymbol\theta}(\mat{x}, \mat{y}_1)
\Big),
\end{align*}
where $\sigma$ is the logistic function.
This sampling-based solution with one noise sample is equivalent to DPO  where the noise sample acts as a dispreferred completion (\eref{eq:dpo}).
This provides a novel interpretation on dispreferred completions in existing PO:
dispreferred completions can be understood as importance samples used to estimate the normalization constant in NLL estimation.

Building on this connection, we can adapt various sampling-based algorithms from the NLL estimation literature to generate dispreferred samples for PO. 
These algorithms aim to improve the accuracy of estimating the normalization constant, thereby improving PO performance.
For instance,
in Proposition \ref{prop: ranking noise contrastive estimation objective},
RNCE suggests random sampling to construct dispreferred completions. 
In the sampling based NLL estimation literature, there exist more advanced strategy than RNCE \cite{olmin2024connection}. So in the next section, we develop a more advanced sampling strategy for PO.

\begin{algorithm}[h]
\caption{MC Kernel}
\label{alg: contrastive divergence kernel}
\algorithmicinput \;\; $\mat{x}$, $\mat{y}_0$
\\
\algorithmicdo \;\; Sample $\{\mat{y}_i\}_{i=1}^{L}$ from $\mu$
\\
\algorithmicdo \;\; Compute $\{ w_i \}_{i=0}^{L}$,
\; $w_i = \frac{\exp 
\big(
\beta
r_{\boldsymbol\theta}(\mat{x}, \mat{y}_i)
\big)} {\sum_{j=0}^L \exp 
\big(
\beta
r_{\boldsymbol\theta}(\mat{x}, \mat{y}_j)
\big)}$
\\
\algorithmicdo \;\; Sample $z \sim {\rm Categorical}\big([w_0, w_1,...,w_{L}]\big)$
\\
\algorithmicoutput \;\; $\mat{y}_z$
\end{algorithm}

\section{Novel Preference Optimization via CD}
\label{sec: preference optimization via contrastive divergence}
Contrastive divergence (CD) uses MCMC methods to estimate the gradient of the log-normalization constant.
CD starts the MCMC sampling from training data rather than a random state, which allows the sampling to converge faster. 
The sampling process involves a small number of MCMC steps (often just one), making it particularly effective for probability models where the normalization constant cannot be easily computed.
CD represents a class of sampling strategies that can be implemented by developing different MC Kernels.
The aforementioned RNCE is a special case of CD with random sampling.
Based on the theoretical foundation that connects PO with sampling-based solutions for NLL estimation,
we first derive a CD algorithm for PO, referred to as MC-PO (Sec. \ref{sec: mcmc preference optimization}).
We then extend MC-PO to an online setting, developing OnMC-PO (Sec. \ref{sec: omcmc preference optimization}).

\subsection{Preference Optimization with MCMC Sampling}
\label{sec: mcmc preference optimization}
To begin with,
we derive the gradient of the NLL estimation in Eq.~\eqref{eq: nll estimation objective function}.
\begin{align}
\label{eq: gradient of mle objective function}
& \nabla_{\boldsymbol\theta} \mathcal{L}^{NLL} (\boldsymbol\theta, \mat{x}, \mat{y})
\\ \nonumber
& =
- \nabla_{\boldsymbol\theta} 
\beta
r_{\boldsymbol\theta}(\mat{x}, \mat{y}_0)
+
\nabla_{\boldsymbol\theta} \log Z_{\boldsymbol\theta}(\mat{x}),
\\ \nonumber
& =
- \nabla_{\boldsymbol\theta} 
\beta
r_{\boldsymbol\theta}(\mat{x}, \mat{y}_0)
+
\mathbb{E}_{p_{\boldsymbol{\theta}}(\mathbf{y} | \mathbf{x})}
\Big[
\nabla_{\boldsymbol\theta} 
\beta
r_{\boldsymbol\theta}(\mat{x}, \mat{y})
\Big].
\end{align}
This gradient term is intractable to compute since it involves an expected value over the probability model $p_{\boldsymbol\theta}$ defined in Eq.~\eqref{eq: mle parameterized policy}.
To address this,
CD applies a MC Kernel $K_{\boldsymbol\theta}(\mat{y}' | \mat{x}, \mat{y})$ to estimate the gradient of the log-normalization constant.
the MC Kernel generates samples with high likelihood from $p_{\boldsymbol\theta}$ via a proposal distribution.

We consider the MC Kernel defined in Algorithm \ref{alg: contrastive divergence kernel}.
Given a proposal distribution $\mu$,
this kernel is initialized at a pair of $\mat{x}$ and $\mat{y}_0$ sampled from $\pi^{\ast}$\footnote{As in Sec. \ref{sec: preference optimization via sampling-based solution}, $\mat{y}_0$ can be replace with a preferred completion.}.
At the initial step of the MCMC chain,
it first generates $L$ samples from the proposal distribution,
then it samples the output $\mat{y}'$ from a softmax distribution computed from the unnormalized model over all $L$ samples.
At the next iteration,
this kernel computation is repeated with an initialization of $\mat{y}'$ being the sampled output from the previous step.
The MC Kernel aims to generate a sample with high estimated reward from a proposal distribution.

\paragraph{Proposed: CD samples hard negatives for PO.}
We first connect CD with RNCE and discuss the sampling strategy suggested by CD.
Then we apply this sampling to PO.

\begin{restatable}{lemma}{cdgeneralizernce}
\label{prop: cd generalizes rnce}
When CD runs the MCMC chain for only a single step,
it shares the same objective function with RNCE in Eq.~\eqref{eq: ranking noise contrastive estimation}.
\end{restatable}
The detailed derivation is in Appendix \ref{sec: cd generalizes rnce}.
Given a true observation $\mat{y}_0$ and noisy samples $\{\mat{y}_i\}_{i=1}^M$, the objective functions of CD and RNCE are equivalent under a special condition. 
However, the MC Kernel from CD, as outlined in Algorithm \ref{alg: contrastive divergence kernel}, suggests to sample in proportion to the reward model. 
This leads to more accurate NLL estimation.
Specifically, the gradient of the log-normalization constant in Eq.~\eqref{eq: gradient of mle objective function} is represented as the expected value over the probability model. 
By sampling in proportion with the reward model, it effectively generates samples with higher probability mass from the probability model, thereby improving the coverage of the distribution in the expected value.

Recall the connection between CD and RNCE established in Sec. \ref{sec: preference optimization via sampling-based solution},
existing PO can be formulated as CD.
CD leads to improved accuracy by sampling proportionally to the reward model,
which suggests to choose hard negatives with high estimated rewards for PO.
Here we show that PO benefits from hard negatives.
\begin{restatable}{lemma}{dpogradient}
\label{prop: dpo gradient}
Let $M=1$,
the gradient of the sampling-based objective in Eq.~\eqref{eq: ranking noise contrastive estimation} can be derived as follows,
\begin{align*}
\nabla_{\boldsymbol\theta} \mathcal{L}^{Sample}(\boldsymbol\theta, \mat{x}, \mat{y}_0)
& =
-
\beta
\sigma
\Big(
\beta
r_{\boldsymbol\theta}(\mat{x}, \mat{y}_1)
-
\beta
r_{\boldsymbol\theta}(\mat{x}, \mat{y}_0)
\Big)
\\ \nonumber
& \;\;\;\;
\nabla_{\boldsymbol\theta}
\Big(
r_{\boldsymbol\theta}(\mat{x}, \mat{y}_0)
-
r_{\boldsymbol\theta}(\mat{x}, \mat{y}_1)
\Big),
\end{align*}
where $\mat{y}_0$ and $\mat{y}_1$ are preferred and dispreferred completions, respectively.
\end{restatable}
The derivation is in Appendix \ref{sec: dpo gradient derivation}.
When the estimated reward for a dispreferred completion exceeds that of a preferred one, this results in a larger gradient magnitude, leading a more effective update of the target policy $\pi_{\boldsymbol\theta}$. 
The MC Kernel, as outlined in Algorithm \ref{alg: contrastive divergence kernel}, aims to sample in proportion to the estimated reward model to achieve this.

\paragraph{Practical Implementation for MC-PO.}
We propose MC-PO as an offline PO algorithm.
For efficiency,
MC-PO runs the MCMC chain for only a single step.
To implement the MC Kernel in Algorithm \ref{alg: contrastive divergence kernel},
we consider a preference dataset that consists of $L$ candidate completions for each input prompt.
During training,
the MC Kernel only computes the weights that depend on the target policy $\pi_{\boldsymbol\theta}$ and samples based on the categorical distribution.
The kernel computation is fast as it is independent from computing the gradient for parameter updates.

\subsection{MCMC for Online Preference Optimization}
\label{sec: omcmc preference optimization}
\paragraph{Online MC-PO leads to an unbiased gradient estimator.}
Having an unbiased gradient estimation is a standard condition to guarantee general convergence of stochastic gradient descent \citep{bottou2018optimization}.
We demonstrate that in an online setting where the true observation is sampled from the probability model, rather than from the target distribution, 
then the CD estimation of the gradient of log-normalization in Eq.~\eqref{eq: gradient of mle objective function} is an unbiased estimator.
\begin{restatable}{proposition}{UnbiasedMLE}
\label{prop: unbiased mle estimator}
Let 
$\hat{Z}_{\boldsymbol\theta}(\mat{x})$
be an estimation of the normalization constant,
\begin{equation*}
\hat{Z}_{\boldsymbol\theta}(\mat{x})
=
\frac{1}{M + 1}
{\sum_{i=0}^M \exp
\Big(
\beta
r_{\boldsymbol\theta}(\mat{x}, \mat{y}_i)
\Big)}.
\end{equation*}
When $\mat{y}_0$ is sampled from the probability model $p_{\boldsymbol\theta}$,
then
\begin{equation*}
\mathbb{E}_{p_{\boldsymbol\theta}(\mat{y}_0 | \mat{x}) 
\mu(\{\mat{y}_i\}_{i=1}^M | \mat{x}))
}
\Big[
\nabla_{\boldsymbol\theta}
\log
\hat{Z}(\mat{x})
\Big]
=
\nabla_{\boldsymbol\theta} \log Z_{\boldsymbol\theta}(\mat{x}),
\end{equation*}
where $\mu(\{\mat{y}_i\}_{i=1}^M | \mat{x})) = \prod_{i=1}^M \mu(\mat{y}_i | \mat{x})$.
\end{restatable}
The detailed derivation is in Appendix \ref{sec: unbiased mle estimator}.
This explains the clear advantage of online methods over offline methods \citep{tang2024understanding}.
Specifically,
online PO algorithms generate preferred completions from the target policy that is proportional to the probability model $p_{\boldsymbol\theta}$,
intends to have an unbiased gradient estimation.

\paragraph{Practical implementation for OnMC-PO.}
As suggested by Proposition \ref{prop: unbiased mle estimator},
it is desirable to sample the preferred completion from the probability model $p_{\boldsymbol\theta}$.
We implement online MC-PO (OnMC-PO) as an extension of MC-PO.
Given a input prompt,
we sample multiple completions from the target policy $\pi_{\boldsymbol\theta}$ and identify the most preferred one as $\mat{y}_0$.
Moreover,
since the policy update at each step is relatively small,
we consider a batched online algorithm \citep{rosset2024direct} where sampling from $\pi_{\boldsymbol\theta}$ is done after a number of gradient updates.

\section{Related Works}
\label{sec: related works}
Aligning LLMs with human preferences has predominately been considered as an RL problem \citep{ouyang2022training}.
However,
the on-policy nature of RLHF necessitates learning a reward model from preference data first,
then maximize the estimated reward model with RL techniques,
leading to a two-stage optimization process \citep{schulman2017proximal}.
Recent developments in preference-based alignment techniques have streamlined this process \citep{rafailov2024direct,azar2024general}.
They enable direct model alignment through a singular loss.
We categorize existing DPO variants as contrastive-based or classification-based approaches according to their objective functions.
Contrastive-based approaches maximize the difference of the predicted likelihoods between preferred and dispreferred completions,
while classification-based approaches conduct maximization on the preferred and minimization on dispreferred completions, respectively.

Some notable contastive-based algorithms include 
DPO \citep{rafailov2024direct}
that is derived from reparametrizing the reward function in RLHF to directly learn a policy from preference data.
IPO \citep{azar2024general} that replaces the logistic loss with a squared loss to address the shortcomings of Bradely-Terry preference modeling in cases where preference data are highly deterministic.
SimPO \citep{meng2024simpo} introduces length regularization on the log-probabilities of both preferred and dispreferred completions, eliminating the need for a reference model. 
RPO \citep{liu2024provably} that derives a superivsed next-word prediction regularization to prevent the decrease of the likelihood to predict preferred completions.
The first classification-based algorithms is
KTO \citep{ethayarajh2024kto} that formulate both maximization and minimization objectives w.r.t. a reference point.
BCO \citep{jung2024binary}
derives the reference point that minimizes the gap with DPO.
NCA \citep{chen2024noise} is derived from noise contrastive estimation for working with reward data \citep{gutmann2010noise}.

In this work, we formulate the alignment problem as sampling-based solutions to solve NLL estimation. 
We first propose the RNCE based sampling as a general PO solution that randomly selects dispreferred completions from a set of candidates.
This solution is similar to InfoNCA \citep{chen2024noise} that generalizes DPO to adopt multiple dispreferred completions. 
Different from InfoNCA, our proposed NLL estimation perspective of model alignment interprets dispreferred completions as the estimative samples to compute the normalization constant, which provides theoretical guidance on developing sampling strategies to choose dispreferred completions for PO.
Based on the NLL estimation formulation, we further develop MC-PO that uses an MCMC kernel to select high-quality dispreferred completions, leading to improved model performance.

\section{Numerical Experiments}
\label{sec: numerical experiments}
In this section, we present main results of our experiments, highlighting the effectiveness of MC-PO and OnMC-PO on various benchmarks (Sec. \ref{sec: main results}) and providing an in-depth understanding on the effect of sampling strategies (Sec. \ref{sec: experiments on sampling strategies}).
More extensive results can be found in Appendix \ref{appendix: extensive experimental results}.

\subsection{Experimental Setup}
\label{sec: experimental setup}
We summarize experimental setting here, more details can be found in Appendix \ref{appendix: experimental setup}.
\paragraph{Models and datasets.}
We perform PO under \textbf{three} different setups: \textbf{ (1) The base setup} considers the
\href{https://huggingface.co/allenai/Llama-3.1-Tulu-3-8B-SFT}{Llama-3.1-8B-SFT} model, which has been fine-tuned using supervised next-word prediction on the TÜLU 3 SFT Mix dataset \citep{lambert2024t},
and {\href{https://huggingface.co/HuggingFaceH4/mistral-7b-sft-beta}{Mistral-7B-SFT}}.
We fine-tune these models on the \href{https://huggingface.co/datasets/berkeley-nest/Nectar}{Nectar} dataset \citep{zhu2023starling}.
The Nectar dataset consists of $7$ ranked completions per input prompt generated by different LLMs,
which creates both high-quality and diverse candidate completions for sampling.
For each input prompt,
we consider the rank-$1$ completion as the preferred completion and subsequently eliminate the rank-$2$ completion to minimize noises in preference pairs. 
From the remaining $5$ candidates, we then randomly select a dispreferred completion.
\textbf{ (2) The instruct setup} uses the off-the-shelf instruction-tuned
\href{https://huggingface.co/meta-llama/Llama-3.1-8B-Instruct}{Llama-3.1-8B-Instruct} model \citep{dubey2024llama} to initialize the target policy $\pi_{\boldsymbol\theta}$.
This model has undergone extensive instruction-tuning processes, making it more expressive compared to the initialization model in the base setup.
We use prompts from the \href{https://huggingface.co/datasets/HuggingFaceH4/ultrafeedback_binarized}{UltraFeedback} dataset \citep{cui2023ultrafeedback} to regenerate the preferred and dispreferred completions using the \href{https://huggingface.co/meta-llama/Llama-3.1-8B-Instruct}{Llama-3.1-8B-Instruct} model.
This makes the instruct setup closer to an on-policy setting \citep{tang2024understanding}.
Specifically,
we generate $6$ completions using temperatures of $0.6$, $0.8$, and $1$ for each input prompt.
Then, we apply the iterative pairwise ranking approach \citep{chen2024towards} with the
\href{https://huggingface.co/meta-llama/Llama-3.1-70B-Instruct}{Llama-3.1-70B-Instruct} to select the most preferred completion and randomly sample a dispreferred completion from remaining candidates.
\textbf{ (3) The batched online setup} is in the middle of the offline and purely online setups \citep{schulman2017proximal, lambert2024t}, striking a balance between efficiency and adaptability.
We equally split the training steps into three batches and regenerate the preference data following the instruct setup using the current model checkpoint.
This approach is more efficient than a purely online setting \citep{qi2024online}, as initializing the inference  is often computationally expensive \citep{kwon2023efficient}.

\paragraph{Training.}
All training jobs are done using full-parameter tuning. 
We fix the effective batch size of $128$ and the number of training epochs of $2$.
Hyperparameter optimization is conducted using $7$ different learning rates.
All results are reported as the average of the final checkpoints across $3$ random seeds, along with the standard deviation,
which can effectively reduce numerical randomness \citep{miller2024adding}.
Each training job is done on a node of $8\cdot$A100 GPUs.

\paragraph{Evaluations.}
To evaluate the performance of aligned models,
we use two popular open-ended instruction-following benchmarks:
AlpacaEval $2$ \citep{dubois2024length} and Arena-Hard \citep{li2024live}.
These benchmarks assess the model's versatile conversational capabilities across a wide range of queries and have been widely adopted by the community.
We focus on winrate as evaluation metric.
Let $N_{\rm cand}$, $N_{\rm base}$ and $N_{\rm tie}$ be the number of candidate wins, baseline wins and ties, respectively.
The adjusted winrate is computed as 
\begin{equation*}
{\rm Winrate}
:=
\frac{N_{\rm cand} + N_{\rm tie} / 2} 
{N_{\rm cand} + N_{\rm base} + N_{\rm tie}}.
\end{equation*}
All winrate-based evaluations are done using \href{https://huggingface.co/mistralai/Mistral-Large-Instruct-2407}{Mistral-Large-Instruct-2407  } as the model judge.

\subsection{Main Results: Comparing with SOTA PO}
\label{sec: main results}
\begin{table*}[h!]
\centering
\begin{tabular}{l|cc|cc|cc}
\hline
Model & \multicolumn{2}{|c|}{\href{https://huggingface.co/HuggingFaceH4/mistral-7b-sft-beta}{Mistral-7B-SFT}} & \multicolumn{2}{|c|}{\href{https://huggingface.co/allenai/Llama-3.1-Tulu-3-8B-SFT}{Llama-3.1-8B-SFT}} & \multicolumn{2}{|c}{\href{https://huggingface.co/meta-llama/Llama-3.1-8B-Instruct}{Llama-3.1-8B-Instruct}} \\
\band Train dataset & \multicolumn{2}{|c|}{Nectar} & \multicolumn{2}{|c|}{Nectar} & \multicolumn{2}{|c}{Ultrafeedback (prompt only)} \\
\band Evaluation & Alpaca & Arena & Alpaca & Arena & Alpaca & Arena \\ \hline
DPO & 25.07($\pm$6.81) & 42.01($\pm$11.88) & 33.74($\pm$2.51) & 60.25($\pm$2.12) & 64.22($\pm$1.01) & 75.88($\pm$0.79) \\
RPO & 15.31($\pm$0.62) & 39.18($\pm$0.49) & 32.50($\pm$0.75) & 59.20($\pm$0.82) & 51.27($\pm$0.50) & 64.74($\pm$0.12) \\
EXO & 21.77($\pm$4.09) & 30.63($\pm$3.55) & 26.48($\pm$3.31) & 52.89($\pm$5.03) & 64.75($\pm$1.72) & 74.93($\pm$0.81) \\
SimPO & 18.62($\pm$2.64) & 48.26($\pm$3.90) & 33.71($\pm$1.41) & 60.69($\pm$1.01) & 54.28($\pm$1.48) & 73.36($\pm$1.38) \\
CPO & 24.27($\pm$0.39) & 49.66($\pm$0.34) & 29.10($\pm$1.01) & 55.25($\pm$0.60) & 65.28($\pm$0.54) & \textcolor{blue}{77.92($\pm$1.78)} \\ \hline
BCO & 23.04($\pm$0.19) & 46.68($\pm$1.62) & 24.96($\pm$1.28) & 58.16($\pm$1.76) & 61.17($\pm$1.27) & 73.45($\pm$0.54) \\
KTO & 22.98($\pm$0.23) & 45.77($\pm$1.85) & 24.50($\pm$1.35) & 53.40($\pm$0.75) & 60.35($\pm$0.67) & 71.19($\pm$0.49) \\
APO & 15.79($\pm$0.78) & 35.94($\pm$0.26) & 21.13($\pm$0.40) & 53.25($\pm$0.82) & 57.54($\pm$0.97) & 70.70($\pm$0.25) \\
SPPO & 12.68($\pm$0.27) & 30.87($\pm$0.67) & 20.26($\pm$0.34) & 53.52($\pm$0.56) & 56.39($\pm$0.58) & 71.73($\pm$0.62) \\
NCA & 17.30($\pm$0.37) & 39.88($\pm$0.80) & 20.46($\pm$0.36) & 53.36($\pm$1.25) & 58.04($\pm$0.42) & 72.40($\pm$0.23) \\ \hline
\band  MC-PO & \textcolor{blue}{30.86($\pm$0.91)} & \textbf{52.75($\pm$2.00)} & 35.84($\pm$0.31) & \textbf{63.77($\pm$0.81)} & 66.90($\pm$0.74) & \textbf{76.71($\pm$0.24)} \\ 
\band OnMC-PO & \textbf{30.52($\pm$0.24)} & \textcolor{blue}{52.90($\pm$0.53)} & \textcolor{blue}{39.70($\pm$0.29)} & \textcolor{blue}{64.21($\pm$0.59)} & \textcolor{blue}{72.63($\pm$0.21)} & \textbf{77.71($\pm$0.85)} \\ \hline
\end{tabular}
\caption{
Performance evaluation of preference-optimized models.
Results are reported as winrate against GPT-4 as baseline.
Each experiment is conducted using three random seeds. 
We report the mean winrate and standard deviation for both \href{https://huggingface.co/datasets/alpaca_eval/alpaca_eval_2}{AlpacaEval $2$} \citep{dubois2024length} and \href{https://huggingface.co/datasets/arena_hard/arena_hard}{Arena-Hard} \citep{li2024live}. 
The models \href{https://huggingface.co/HuggingFaceH4/mistral-7b-sft-beta}{Mistral-7B-SFT} and \href{https://huggingface.co/allenai/Llama-3.1-Tulu-3-8B-SFT}{Llama-3.1-8B-SFT} are trained using the \href{https://huggingface.co/datasets/berkeley-nest/Nectar}{Nectar} dataset, 
\href{https://huggingface.co/meta-llama/Llama-3.1-8B-Instruct}{Llama-3.1-8B-Instruct} is trained using prompts from the \href{https://huggingface.co/datasets/HuggingFaceH4/ultrafeedback_binarized}{UltraFeedback} dataset and self-generated completions.
The highest scores are highlighted in blue, 
and scores within one standard deviation of the highest are boldfaced.
}
\label{table: winrate model performance evaluation llama}
\end{table*}

We first compare MC-PO with existing offline PO algorithms,
then demonstrate that OnMC-PO further improves alignment performance.
We categorize existing baselines as contrastive and classification-based approaches based on their objective functions.
Specifically,
contrastive-based algorithms include 
\textbf{DPO} \citep{rafailov2024direct},
\textbf{RPO} \citep{liu2024provably},
\textbf{EXO} \citep{jitowards},
\textbf{SimPO} \citep{meng2024simpo} and
\textbf{CPO} \citep{xucontrastive}.
Classification-based algorithms include
\textbf{BCO} \citep{jung2024binary},
\textbf{KTO} \citep{ethayarajh2024kto},
\textbf{APO} \citep{d2024anchored},
\textbf{SPPO} \citep{wu2024self}
and \textbf{NCA} \citep{chen2024noise}.
Details on baseline algorithms can be found in Appendix \ref{appendix: baseline preference optimization algorithms}.

The main results are summarized in Table \ref{table: winrate model performance evaluation llama}. 
MC-PO outperforms existing baselines in five out of six studies. 
Notably, in the base setup, using the {\href{https://huggingface.co/HuggingFaceH4/mistral-7b-sft-beta}{Mistral-7B-SFT}} model, 
MC-PO outperforms DPO by 
$4.5\%$ and $9\%$ on Alpaca-Eval and Arena, respectively. 
Using the {\href{https://huggingface.co/allenai/Llama-3.1-Tulu-3-8B-SFT}{Llama-3.1-8B-SFT}} model, MC-PO leads to winrates of $35.84\%$ and $63.77\%$ on Alpaca-Eval and Arena, respectively. 
In the instruct setup, as all candidate completions are sampled from the {\href{https://huggingface.co/meta-llama/Llama-3.1-8B-Instruct}{Llama-3.1-8B-Instruct}} model, sampling from these candidates proves less effective due to the low diversity in the candidate set. 
Consequently, MC-PO shows less improvement compared to existing baselines.
When MC-PO is extended to online setting based on the batched online setup,
OnMC-PO results in further improvement.
With the {\href{https://huggingface.co/meta-llama/Llama-3.1-8B-Instruct}{Llama-3.1-8B-Instruct}} model,
OnMC-PO leads to a winrate of $72.63 \%$ on Alpaca, outperforming existing baselines.

\subsection{Analysis of Sampling Strategies in MC-PO}
\label{sec: experiments on sampling strategies}
\begin{figure}[th]
\centering
\begin{minipage}{.99\columnwidth}
\centering
    \includegraphics[width = \linewidth]{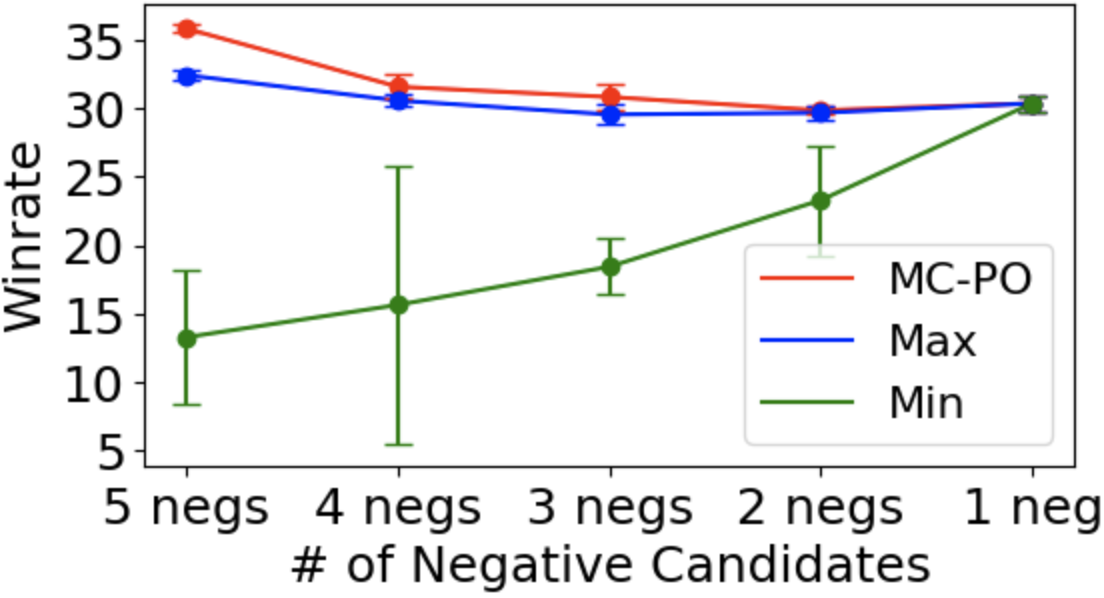}
    (a): Alpaca-Eval
\end{minipage}
\begin{minipage}{.99\columnwidth}
\centering
    \includegraphics[width = \linewidth]{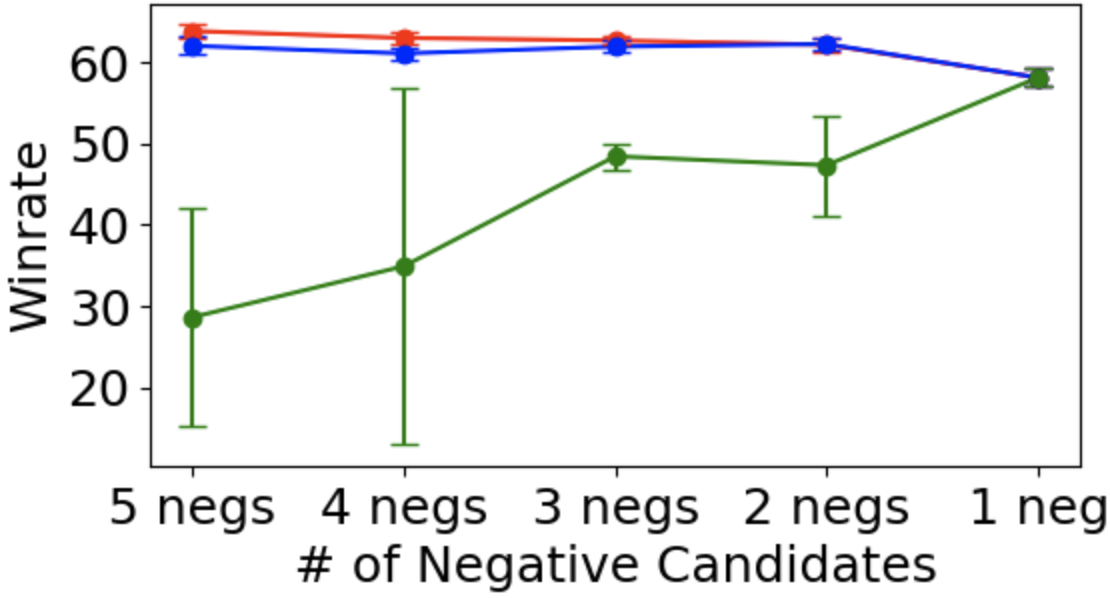}
    (b): Arena
\end{minipage}
\caption{
Winrate evaluation of the optimized \href{https://huggingface.co/allenai/Llama-3.1-Tulu-3-8B-SFT}{Llama-3.1-8B-SFT} model using MC-PO, versus its Max, and Min sampling based variants. 
Five modified \href{https://huggingface.co/datasets/berkeley-nest/Nectar}{Nectar} datasets are used for training. 
$x$ negs represents that the training dataset contains $x$ negative candidates for each input prompt. 
For example, the $3$ negs dataset is constructed by removing rank-$2$ and rank-$3$ completions from the \href{https://huggingface.co/datasets/berkeley-nest/Nectar}{Nectar} dataset.
}
\label{fig: ablation sampling strategies}
\end{figure}

We also study how varying the sampling strategies in MC-PO impact the PO performance as the quality of sampled preference dataset gets changed.
We develop Max and Min samplings as variants of MC-PO based on the MCMC kernel defined in Algorithm \ref{alg: contrastive divergence kernel}. Here, 
Max (Min) sampling variant outputs the candidate with maximum (minimum) weight among all candidates,
where the weight is calculated as 
$
w_i = \frac{\exp 
\big(
\beta
r_{\boldsymbol\theta}(\mat{x}, \mat{y}_i)
\big)} {\sum_{j=0}^L \exp 
\big(
\beta
r_{\boldsymbol\theta}(\mat{x}, \mat{y}_j)
\big)}
$.
Moreover,
we construct preference datasets with varying candidate qualities. For instance, based on the \href{https://huggingface.co/datasets/berkeley-nest/Nectar}{Nectar} dataset, we progressively remove highly ranked completions for each input prompt. 
The first dataset excludes the rank-$2$ completion, while the second dataset excludes both the rank-$2$ and rank-$3$ completions, resulting in diminished candidate quality.
The results are summarized in Fig.~\ref{fig: ablation sampling strategies},
from which we observe the following insights:

\textbf{(1) Sampling from the MCMC kernel yields optimal performance.}
MC-PO achieves a balance between exploitation (sampling according to the categorical distribution), and exploration  
(retaining probabilities for choosing alternative candidates). As detailed in Sec. \ref{sec: mcmc preference optimization}, this approach accurately estimates the gradient of the log-normalization constant, which in turn, leads to improved  performance.

\textbf{(2) Min-based variant leads to low performance and high variance.}
From the NLL estimation viewpoint of PO, dispreferred samples are used for estimating the gradient of the log-normalization constant. 
CD proves that hard negatives yield a more accurate gradient estimation. 
The Min sampling variant, instead, is deemed as the worst sampling strategy according to CD, 
leading to inaccurate gradient estimations and therefore resulting in lower model performance and increased variance.

\textbf{(3) MC-PO correlates with the quality of candidates proportionally.}
When the preference dataset includes five high-quality candidates for each input prompt (referred to as $5$ negs), both MC-PO and Max strategies yield the best model performance. 
However, as high-quality completions are eliminated from the candidate sets, the performance of models optimized with MC-PO and Max variant declines due to the reduced quality of candidates. 
When there is only one candidate per prompt (referred to as $1$ neg), all three sampling strategies are equivalent.

\subsection{Data and Ablation Analysis of MC-PO}
\label{sec:expdata}

\textbf{MC-PO is robust against noisy samples: }
We demonstrate that MC-PO maintains high model performance even when noisy samples are included in the candidate set. 
Differently, it has been proven that when the edit distance between pairs of completions is low, DPO leads to a reduction in the model’s likelihood of preferring the desired completions \citep{pal2024smaug}. 

For experimental setup, we consider the processed \href{https://huggingface.co/datasets/berkeley-nest/Nectar}{Nectar} dataset and inject a noise sample into the candidate set by randomly switching two tokens of the preferred completion for each input prompt. 
As shown in Table \ref{table: model performance when noise samples are included},
due to the small edit distance between all preference pairs, DPO($-$), which applies the noise sample as dispreferred completion, leads to a degenerated model. 
DPO, which randomly selects a dispferred completion, is impacted by the noise injection.
MC-PO, however,  samples a dispreferred completion based on the log-probabilities of all candidates, chooses semantically hard negatives instead of syntactically similar negatives with small edit distances.

\paragraph{MC-PO benefits from sampling more negatives.}

\begin{table}[th]
\centering
\begin{tabular}{l|cc}
\hline
Model & \multicolumn{2}{|c}{\href{https://huggingface.co/allenai/Llama-3.1-Tulu-3-8B-SFT}{Llama-3.1-8B}} \\
Evaluation & Alpaca & Arena \\ \hline
DPO($-$) & 1.08($\pm$0.6) & 3.17($\pm$0.9)
\\
DPO & 23.62($\pm$2.81) & 50.51($\pm$5.59)
\\
\band MC-PO & \textcolor{blue}{28.98($\pm$1.34)} & \textcolor{blue}{58.09($\pm$2.63)} \\ \hline
\end{tabular}
\caption{
Evaluation of models trained with DPO and MC-PO when noise samples are included in the dispreferred candidate set. 
DPO($-$) uses the noise sample as a dispreferred completion, 
DPO selects a dispreferred completion at random, 
and MC-PO samples from a candidate set that includes the noise sample.
}
\label{table: model performance when noise samples are included}
\end{table}
\begin{table}[th]
\centering
\begin{tabular}{l|ccc}
\hline
\multicolumn{4}{c}{\href{https://huggingface.co/datasets/berkeley-nest/Nectar}{Nectar} ~/~ 
\href{https://huggingface.co/allenai/Llama-3.1-Tulu-3-8B-SFT}{Llama-3.1-8B-SFT}
}
\\ \hline
Alpaca & $M=1$ & $M=2$ & $M=3$ \\
RNCE & 33.74(2.51) & 33.73(0.49) & 34.36(0.56)
\\
\band MC-PO & \textcolor{blue}{35.84(0.31)} & \textcolor{blue}{36.73(0.59)} & \textcolor{blue}{37.40(0.13)} \\ 
\hline
Arena & $M=1$ & $M=2$ & $M=3$ \\
RNCE & 60.25(2.12) & 61.53(0.29) & 61.16(0.69)
\\
\band MC-PO & \textcolor{blue}{63.77(0.81)} & \textcolor{blue}{64.53(0.60)} & \textcolor{blue}{66.16(0.13)} \\ 
\hline
\end{tabular}
\caption{
Performance comparison of preference-optimized models using RNCE and MC-PO with multiple dispreferred samples.
}
\label{table: model performance sampling strategy more dispreferred completions}
\end{table}
In Table \ref{table: model performance sampling strategy more dispreferred completions},
we examine the performance of MC-PO and RNCE when the number of dispreferred samples is greater than $1$. 
It is evident that RNCE (who uses random sampling) does not achieve notable improvements when having more  dispreferred samples. Conversely, MC-PO (who utilizes an MCMC kernel for sampling) consistently demonstrates improved performance when the number of dispreferred samples increases.

\textbf{MC-PO versus augmented training dataset.}
MC-PO optimizes models on a specialized data format where each input prompt is paired with a preferred completion alongside multiple dispreferred completions. 
Utilizing this data format, an alternative method can be augmenting the training dataset by pairing each preferred completion with each of its corresponding dispreferred completions. 
For instance, in the processed \href{https://huggingface.co/datasets/berkeley-nest/Nectar}{Nectar} dataset, where each prompt contains $5$ candidate completions, we can increase the dataset size four-fold by this augmentation approach. 
Subsequently, we implement DPO on the augmented dataset, as an alternative to compare with MC-PO. 
Recall in Table \ref{table: winrate model performance evaluation llama} that the {\href{https://huggingface.co/allenai/Llama-3.1-Tulu-3-8B-SFT}{Llama-3.1-8B-SFT}} model trained with MC-PO achieves winrates of $35.84(\pm0.31)$ and $63.77(\pm0.81)$ on Alpaca-Eval and Arena, respectively. 
This alternative solution, which increases training time by $4$X, leads to winrates of $34.18(\pm1.26)$ and $59.62(\pm1.04)$ on Alpaca-Eval and Arena, respectively.

\textbf{OnMC-PO versus Online DPO.}
We compare OnMC-PO with online DPO \citep{guo2024direct} that applies random sampling to choose a disprerred completion.
Both algorithms generate completion in an batched online setting.
With the \href{https://huggingface.co/meta-llama/Llama-3.1-8B-Instruct}{Llama-3.1-8B-Instruct} model.
Online DPO achieves winrates of $72.63 \%$ and $73.28 \%$ on Alpaca-eval and Arena, respectively,
On Arena,
OnMC-PO outperforms online DPO by a large margin, achieve a winrate of $77.71 \%$.

\vspace{-2mm}
\section{Conclusion, Limitations and Future Work.}
\vspace{-2mm}
This paper formulates the alignment problem via NLL estimation and propose sampling-based solutions for better PO. Compared to DPO, our CD based MC-PO adds approximately $30 \%$ training time because of the computations required by the MCMC kernel. Also we run MC-PO's MCMC kernel for a single step and this is typically sub-optimal because executing the MCMC chain for multiple steps is essential to acquire high-quality samples \citep{hinton2002training}, which will further adds more  computational overhead to PO training. As future research, we aim to showcase the benefits of utilizing a multi-step MCMC based PO solutions and will develop more efficient training algorithms to speed up these algorithms.

\newpage
\bibliography{00-main}

\begin{thebibliography}{38}
\providecommand{\natexlab}[1]{#1}
\providecommand{\url}[1]{\texttt{#1}}
\expandafter\ifx\csname urlstyle\endcsname\relax
  \providecommand{\doi}[1]{doi: #1}\else
  \providecommand{\doi}{doi: \begingroup \urlstyle{rm}\Url}\fi

\bibitem[Azar et~al.(2024)Azar, Guo, Piot, Munos, Rowland, Valko, and
  Calandriello]{azar2024general}
Azar, M.~G., Guo, Z.~D., Piot, B., Munos, R., Rowland, M., Valko, M., and
  Calandriello, D.
\newblock A general theoretical paradigm to understand learning from human
  preferences.
\newblock In \emph{International Conference on Artificial Intelligence and
  Statistics}, pp.\  4447--4455. PMLR, 2024.

\bibitem[Bottou et~al.(2018)Bottou, Curtis, and
  Nocedal]{bottou2018optimization}
Bottou, L., Curtis, F.~E., and Nocedal, J.
\newblock Optimization methods for large-scale machine learning.
\newblock \emph{SIAM review}, 60\penalty0 (2):\penalty0 223--311, 2018.

\bibitem[Bradley \& Terry(1952)Bradley and Terry]{bradley1952rank}
Bradley, R.~A. and Terry, M.~E.
\newblock Rank analysis of incomplete block designs: I. the method of paired
  comparisons.
\newblock \emph{Biometrika}, 39\penalty0 (3/4):\penalty0 324--345, 1952.

\bibitem[Chen et~al.(2024{\natexlab{a}})Chen, He, Yuan, Cui, Su, and
  Zhu]{chen2024noise}
Chen, H., He, G., Yuan, L., Cui, G., Su, H., and Zhu, J.
\newblock Noise contrastive alignment of language models with explicit rewards.
\newblock In \emph{The Thirty-eighth Annual Conference on Neural Information
  Processing Systems}, 2024{\natexlab{a}}.
\newblock URL \url{https://openreview.net/forum?id=KwRLDkyVOl}.

\bibitem[Chen et~al.(2024{\natexlab{b}})Chen, Liu, Zhu, Du, and
  Qi]{chen2024towards}
Chen, Z., Liu, F., Zhu, J., Du, W., and Qi, Y.
\newblock Towards improved preference optimization pipeline: from data
  generation to budget-controlled regularization.
\newblock \emph{arXiv preprint arXiv:2411.05875}, 2024{\natexlab{b}}.

\bibitem[Cui et~al.(2024)Cui, Yuan, Ding, Yao, Zhu, Ni, Xie, Liu, and
  Sun]{cui2023ultrafeedback}
Cui, G., Yuan, L., Ding, N., Yao, G., Zhu, W., Ni, Y., Xie, G., Liu, Z., and
  Sun, M.
\newblock Ultrafeedback: Boosting language models with high-quality feedback,
  2024.
\newblock URL \url{https://openreview.net/forum?id=pNkOx3IVWI}.

\bibitem[D'Oosterlinck et~al.(2024)D'Oosterlinck, Xu, Develder, Demeester,
  Singh, Potts, Kiela, and Mehri]{d2024anchored}
D'Oosterlinck, K., Xu, W., Develder, C., Demeester, T., Singh, A., Potts, C.,
  Kiela, D., and Mehri, S.
\newblock Anchored preference optimization and contrastive revisions:
  Addressing underspecification in alignment.
\newblock \emph{CoRR}, 2024.

\bibitem[Dubey et~al.(2024)Dubey, Jauhri, Pandey, Kadian, Al-Dahle, Letman,
  Mathur, Schelten, Yang, Fan, et~al.]{dubey2024llama}
Dubey, A., Jauhri, A., Pandey, A., Kadian, A., Al-Dahle, A., Letman, A.,
  Mathur, A., Schelten, A., Yang, A., Fan, A., et~al.
\newblock The llama 3 herd of models.
\newblock \emph{arXiv preprint arXiv:2407.21783}, 2024.

\bibitem[Dubois et~al.(2024)Dubois, Galambosi, Liang, and
  Hashimoto]{dubois2024length}
Dubois, Y., Galambosi, B., Liang, P., and Hashimoto, T.~B.
\newblock Length-controlled alpacaeval: A simple way to debias automatic
  evaluators.
\newblock \emph{arXiv preprint arXiv:2404.04475}, 2024.

\bibitem[Ethayarajh et~al.(2024)Ethayarajh, Xu, Muennighoff, Jurafsky, and
  Kiela]{ethayarajh2024kto}
Ethayarajh, K., Xu, W., Muennighoff, N., Jurafsky, D., and Kiela, D.
\newblock Kto: Model alignment as prospect theoretic optimization.
\newblock \emph{arXiv preprint arXiv:2402.01306}, 2024.

\bibitem[Go et~al.(2023)Go, Korbak, Kruszewski, Rozen, Ryu, and
  Dymetman]{go2023aligning}
Go, D., Korbak, T., Kruszewski, G., Rozen, J., Ryu, N., and Dymetman, M.
\newblock Aligning language models with preferences through f-divergence
  minimization.
\newblock In \emph{Proceedings of the 40th International Conference on Machine
  Learning}, pp.\  11546--11583, 2023.

\bibitem[Guo et~al.(2024)Guo, Zhang, Liu, Liu, Khalman, Llinares, Rame,
  Mesnard, Zhao, Piot, et~al.]{guo2024direct}
Guo, S., Zhang, B., Liu, T., Liu, T., Khalman, M., Llinares, F., Rame, A.,
  Mesnard, T., Zhao, Y., Piot, B., et~al.
\newblock Direct language model alignment from online ai feedback.
\newblock \emph{arXiv preprint arXiv:2402.04792}, 2024.

\bibitem[Gutmann \& Hyv{\"a}rinen(2010)Gutmann and
  Hyv{\"a}rinen]{gutmann2010noise}
Gutmann, M. and Hyv{\"a}rinen, A.
\newblock Noise-contrastive estimation: A new estimation principle for
  unnormalized statistical models.
\newblock In \emph{Proceedings of the thirteenth international conference on
  artificial intelligence and statistics}, pp.\  297--304. JMLR Workshop and
  Conference Proceedings, 2010.

\bibitem[Hinton(2002)]{hinton2002training}
Hinton, G.~E.
\newblock Training products of experts by minimizing contrastive divergence.
\newblock \emph{Neural computation}, 14\penalty0 (8):\penalty0 1771--1800,
  2002.

\bibitem[Ji et~al.(2024)Ji, Lu, Niu, Ke, Wang, Zhu, Tang, and Huang]{jitowards}
Ji, H., Lu, C., Niu, Y., Ke, P., Wang, H., Zhu, J., Tang, J., and Huang, M.
\newblock Towards efficient and exact optimization of language model alignment.
\newblock \emph{arXiv preprint arXiv:2402.00856}, 2024.

\bibitem[Jung et~al.(2024)Jung, Han, Nam, and On]{jung2024binary}
Jung, S., Han, G., Nam, D.~W., and On, K.-W.
\newblock Binary classifier optimization for large language model alignment.
\newblock \emph{arXiv preprint arXiv:2404.04656}, 2024.

\bibitem[Kwon et~al.(2023)Kwon, Li, Zhuang, Sheng, Zheng, Yu, Gonzalez, Zhang,
  and Stoica]{kwon2023efficient}
Kwon, W., Li, Z., Zhuang, S., Sheng, Y., Zheng, L., Yu, C.~H., Gonzalez, J.~E.,
  Zhang, H., and Stoica, I.
\newblock Efficient memory management for large language model serving with
  pagedattention.
\newblock In \emph{Proceedings of the ACM SIGOPS 29th Symposium on Operating
  Systems Principles}, 2023.

\bibitem[Lambert et~al.(2024)Lambert, Morrison, Pyatkin, Huang, Ivison,
  Brahman, Miranda, Liu, Dziri, Lyu, et~al.]{lambert2024t}
Lambert, N., Morrison, J., Pyatkin, V., Huang, S., Ivison, H., Brahman, F.,
  Miranda, L. J.~V., Liu, A., Dziri, N., Lyu, S., et~al.
\newblock T$\backslash$" ulu 3: Pushing frontiers in open language model
  post-training.
\newblock \emph{arXiv preprint arXiv:2411.15124}, 2024.

\bibitem[Li et~al.(2024)Li, Chiang, Frick, Dunlap, Zhu, Gonzalez, and
  Stoica]{li2024live}
Li, T., Chiang, W.-L., Frick, E., Dunlap, L., Zhu, B., Gonzalez, J.~E., and
  Stoica, I.
\newblock From live data to high-quality benchmarks: The arena-hard pipeline,
  2024.

\bibitem[Liu et~al.(2024)Liu, Lu, Zhang, Liu, Guo, Yang, Blanchet, and
  Wang]{liu2024provably}
Liu, Z., Lu, M., Zhang, S., Liu, B., Guo, H., Yang, Y., Blanchet, J., and Wang,
  Z.
\newblock Provably mitigating overoptimization in rlhf: Your sft loss is
  implicitly an adversarial regularizer.
\newblock \emph{arXiv preprint arXiv:2405.16436}, 2024.

\bibitem[Meng et~al.(2024)Meng, Xia, and Chen]{meng2024simpo}
Meng, Y., Xia, M., and Chen, D.
\newblock Simpo: Simple preference optimization with a reference-free reward.
\newblock \emph{arXiv preprint arXiv:2405.14734}, 2024.

\bibitem[Miller(2024)]{miller2024adding}
Miller, E.
\newblock Adding error bars to evals: A statistical approach to language model
  evaluations.
\newblock \emph{arXiv preprint arXiv:2411.00640}, 2024.

\bibitem[Naesseth et~al.(2024)Naesseth, Lindsten, and
  Schön]{naesseth2024elementssequentialmontecarlo}
Naesseth, C.~A., Lindsten, F., and Schön, T.~B.
\newblock Elements of sequential monte carlo, 2024.
\newblock URL \url{https://arxiv.org/abs/1903.04797}.

\bibitem[Olmin et~al.(2024)Olmin, Lindqvist, Svensson, and
  Lindsten]{olmin2024connection}
Olmin, A., Lindqvist, J., Svensson, L., and Lindsten, F.
\newblock On the connection between noise-contrastive estimation and
  contrastive divergence.
\newblock In \emph{International Conference on Artificial Intelligence and
  Statistics}, pp.\  3016--3024. PMLR, 2024.

\bibitem[Ouyang et~al.(2022)Ouyang, Wu, Jiang, Almeida, Wainwright, Mishkin,
  Zhang, Agarwal, Slama, Ray, et~al.]{ouyang2022training}
Ouyang, L., Wu, J., Jiang, X., Almeida, D., Wainwright, C., Mishkin, P., Zhang,
  C., Agarwal, S., Slama, K., Ray, A., et~al.
\newblock Training language models to follow instructions with human feedback.
\newblock \emph{Advances in neural information processing systems},
  35:\penalty0 27730--27744, 2022.

\bibitem[Pal et~al.(2024)Pal, Karkhanis, Dooley, Roberts, Naidu, and
  White]{pal2024smaug}
Pal, A., Karkhanis, D., Dooley, S., Roberts, M., Naidu, S., and White, C.
\newblock Smaug: Fixing failure modes of preference optimisation with
  dpo-positive.
\newblock \emph{arXiv preprint arXiv:2402.13228}, 2024.

\bibitem[Peters \& Schaal(2007)Peters and Schaal]{peters2007reinforcement}
Peters, J. and Schaal, S.
\newblock Reinforcement learning by reward-weighted regression for operational
  space control.
\newblock In \emph{Proceedings of the 24th international conference on Machine
  learning}, pp.\  745--750, 2007.

\bibitem[Qi et~al.(2024)Qi, Li, Li, Gao, Zhang, and Zhou]{qi2024online}
Qi, B., Li, P., Li, F., Gao, J., Zhang, K., and Zhou, B.
\newblock Online dpo: Online direct preference optimization with fast-slow
  chasing.
\newblock \emph{arXiv preprint arXiv:2406.05534}, 2024.

\bibitem[Rafailov et~al.(2024)Rafailov, Sharma, Mitchell, Manning, Ermon, and
  Finn]{rafailov2024direct}
Rafailov, R., Sharma, A., Mitchell, E., Manning, C.~D., Ermon, S., and Finn, C.
\newblock Direct preference optimization: Your language model is secretly a
  reward model.
\newblock \emph{Advances in Neural Information Processing Systems}, 36, 2024.

\bibitem[Rosset et~al.(2024)Rosset, Cheng, Mitra, Santacroce, Awadallah, and
  Xie]{rosset2024direct}
Rosset, C., Cheng, C.-A., Mitra, A., Santacroce, M., Awadallah, A., and Xie, T.
\newblock Direct nash optimization: Teaching language models to self-improve
  with general preferences.
\newblock \emph{arXiv preprint arXiv:2404.03715}, 2024.

\bibitem[Schulman et~al.(2017)Schulman, Wolski, Dhariwal, Radford, and
  Klimov]{schulman2017proximal}
Schulman, J., Wolski, F., Dhariwal, P., Radford, A., and Klimov, O.
\newblock Proximal policy optimization algorithms.
\newblock \emph{arXiv preprint arXiv:1707.06347}, 2017.

\bibitem[Skalse et~al.(2022)Skalse, Howe, Krasheninnikov, and
  Krueger]{skalse2022defining}
Skalse, J. M.~V., Howe, N.~H., Krasheninnikov, D., and Krueger, D.
\newblock Defining and characterizing reward gaming.
\newblock In \emph{Advances in Neural Information Processing Systems}, 2022.

\bibitem[Tang et~al.(2024)Tang, Guo, Zheng, Calandriello, Cao, Tarassov, Munos,
  Pires, Valko, Cheng, et~al.]{tang2024understanding}
Tang, Y., Guo, D.~Z., Zheng, Z., Calandriello, D., Cao, Y., Tarassov, E.,
  Munos, R., Pires, B.~{\'A}., Valko, M., Cheng, Y., et~al.
\newblock Understanding the performance gap between online and offline
  alignment algorithms.
\newblock \emph{arXiv preprint arXiv:2405.08448}, 2024.

\bibitem[Tunstall et~al.(2023)Tunstall, Beeching, Lambert, Rajani, Rasul,
  Belkada, Huang, von Werra, Fourrier, Habib, et~al.]{tunstall2023zephyr}
Tunstall, L., Beeching, E., Lambert, N., Rajani, N., Rasul, K., Belkada, Y.,
  Huang, S., von Werra, L., Fourrier, C., Habib, N., et~al.
\newblock Zephyr: Direct distillation of lm alignment.
\newblock \emph{arXiv preprint arXiv:2310.16944}, 2023.

\bibitem[Wu et~al.(2024)Wu, Sun, Yuan, Ji, Yang, and Gu]{wu2024self}
Wu, Y., Sun, Z., Yuan, H., Ji, K., Yang, Y., and Gu, Q.
\newblock Self-play preference optimization for language model alignment.
\newblock \emph{arXiv preprint arXiv:2405.00675}, 2024.

\bibitem[Xu et~al.(2024)Xu, Sharaf, Chen, Tan, Shen, Van~Durme, Murray, and
  Kim]{xucontrastive}
Xu, H., Sharaf, A., Chen, Y., Tan, W., Shen, L., Van~Durme, B., Murray, K., and
  Kim, Y.~J.
\newblock Contrastive preference optimization: Pushing the boundaries of llm
  performance in machine translation.
\newblock \emph{arXiv preprint arXiv:2401.08417}, 2024.

\bibitem[Zhu et~al.(2023)Zhu, Frick, Wu, Zhu, and Jiao]{zhu2023starling}
Zhu, B., Frick, E., Wu, T., Zhu, H., and Jiao, J.
\newblock Starling-7b: Improving llm helpfulness \& harmlessness with rlaif,
  2023.

\bibitem[Ziegler et~al.(2019)Ziegler, Stiennon, Wu, Brown, Radford, Amodei,
  Christiano, and Irving]{ziegler2019fine}
Ziegler, D.~M., Stiennon, N., Wu, J., Brown, T.~B., Radford, A., Amodei, D.,
  Christiano, P., and Irving, G.
\newblock Fine-tuning language models from human preferences.
\newblock \emph{arXiv preprint arXiv:1909.08593}, 2019.

\end{thebibliography}
\bibliographystyle{icml2025}

\newpage
\appendix
\onecolumn

\section{Theoretical Derivations}
\label{sec: theoretical derivations}

\subsection{Background on Preference Optimization}
\label{sec: background on rlhf}
\paragraph{RLHF formulation.}
Language models have shown impressive capabilities in the past few years by generating diverse and compelling text from human input prompts.
One of the key ingredients behind the success of language models is post-training alignment.
Reinforcement learning from human feedback (RLHF) aims to align a target policy $\pi_{\boldsymbol\theta}$ with human preference based on an reward model $r(\mat{x}, \mat{y})$ that approximates human judgements.
This optimization problem is formulated as 
\begin{equation*}
\max \limits_{\pi_{\boldsymbol\theta}}
\mathbb{E}_{{\mat{x} \sim p,\; \mat{y} \sim \pi_\theta(\cdot|\mat{x})}}
[r(\mat{x}, \mat{y})]
-
\beta \cdot KL[\pi_{\boldsymbol\theta}(\mat{y} | \mat{x}) || \pi_{\rm ref}(\mat{y} | \mat{x})],
\end{equation*}
where $\beta$ is a hyper-parameter controlling the deviation from the base reference policy $\pi_{\rm ref}$.
The added constraint is important,
as it prevents the model from deviating too far from the distribution on which the reward is accurate,
as well as maintaining the generation diversity and preventing mode-collapse to single high-reward answers.

Prior works \citep{go2023aligning, peters2007reinforcement} prove that the optimal solution to the KL-constrained reward maximization objective takes the following form,
\begin{equation*}
\pi^{\ast}(\mat{y} | \mat{x})
=
\frac{1} {Z(x)}
\pi_{\rm ref}(y | x) 
\exp 
\Big(
\frac{1} {\beta} r(x, y)
\Big),
\end{equation*}
where $Z(x)$ is the partition function that normalizes the probability.
\begin{equation*}
Z(x)
=
\sum_{\mat{y}}
\pi_{\rm ref}(y | x) 
\exp 
\Big(
\frac{1} {\beta} r(x, y)
\Big).
\end{equation*}
Since the space of model completions is combinatorilly large, computing $Z(x)$ exactly is often infeasible.

\paragraph{Direct preference optimization \citep{rafailov2024direct}.}
Direct preference optimization (DPO) enables to directly optimize a language model to adhere to human preferences,
without explicit reward modeling or reinforcement learning.
Specifically,
DPO reparameterizes the reward model in terms of the target policy,
which enables directly optimizing the target policy from preference dataset.
To begin with,
we can rearrange the close-form solution of RLHF to express the reward function in terms of its corresponding optimal policy, the reference policy,
and the partition function,
\begin{equation*}
r(\mat{x}, \mat{y})
=
\beta 
\log
\frac{\pi^{\ast}(\mat{y} | \mat{x})}
{\pi_{\rm ref}(\mat{y} | \mat{x})}
+
\beta
\log Z(\mat{x}).
\end{equation*}
Substituting this into the Bradley-Terry preference model \citep{bradley1952rank} yields
\begin{equation*}
    p^{*}(\mat{y}_0 \succ \mat{y}_1 | \mat{x})
    = \frac{1}{1 + \exp\left( 
        \beta \log \frac{\pi^{*}(\mat{y}_1 | \mat{x})}{\pi_{\rm ref}(\mat{y}_1 | \mat{x})}
        - \beta \log \frac{\pi^{*}(\mat{y}_0 | \mat{x})}{\pi_{\rm ref}(\mat{y}_0 | \mat{x})}
    \right)},
\end{equation*}
where $\mat{y}_0$ and $\mat{y}_1$ denote preferred and dispreferred completions, respectively.

DPO reformulates this as a maximum likelihood objective over a preference dataset $\mathcal{D} = \{(\mat{x}^{(i)}, \mat{y}_0^{(i)}, \mat{y}_1^{(i)})\}$, leading to:
\begin{equation*}
    \pi_{\boldsymbol{\theta}}^{*}
    = \arg\min_{\pi_{\boldsymbol{\theta}}} 
    \mathbb{E}_{(\mat{x},\mat{y}_0,\mat{y}_1) \sim \mathcal{D}}
    \left[ 
        -\log \sigma\left( 
            \beta \log \frac{\pi_{\boldsymbol{\theta}}(\mat{y}_0 | \mat{x})}{\pi_{\rm ref}(\mat{y}_0 | \mat{x})} 
            - \beta \log \frac{\pi_{\boldsymbol{\theta}}(\mat{y}_1 | \mat{x})}{\pi_{\rm ref}(\mat{y}_1 | \mat{x})}
        \right)
    \right],
\end{equation*}
where $\sigma$ denotes the logistic function. 
This approach enables direct optimization of language models using preference data while avoiding instabilities associated with RL-based training.

\subsection{Proof of Proposition \ref{prop: ranking noise contrastive estimation objective}}
\label{sec: RNCE objective function derivation}
\rnce*

\begin{proof}
The Ranking Noise Contrastive Estimation (RNCE) is based on a multi-class classification problem with a single observation point and multiple noise ones from a proposal distribution.

Suppose we have $\mat{y}_0 \sim \pi^{\ast}(\mat{y} | \mat{x})$,
and noise samples $\mat{y}_i \sim \mu(\mat{y} | \mat{x})$, for $i = 1,2,...,M$.
We use $\{ \mat{y}_i \}_{i=0}^M$ to denote all samples.
Let the variable $z \in \{0, 1, ..., M \}$ denote the index of the observation sampled from $\pi^{\ast}$,
and assume that all outcomes are equally probable a priori,
e.g.,
$P(z=i) = 1 / (M + 1)$, for $i = 0, 1,...,M$.
Conditioned on $\{ \mat{y}_i \}_{i=0}^M$, we want the probability model to maximize the posterior probability of identifying $z = 0$ (the probability of identifying the index corresponding to the observation sampled from $\pi^{\ast}$).

\begin{equation*}
\min \limits_{\boldsymbol\theta}
- \log P(z = 0 | \mat{x}, \{ \mat{y}_i \}_{i=0}^M).
\end{equation*}

According to the Bayes rule and the law of total probability,
\begin{equation*}
P(z = 0 | \mat{x}, \{ \mat{y}_i \}_{i=0}^M)
=
\frac{P(\{ \mat{y}_i \}_{i=0}^M | \mat{x}, z = 0) P(z = 0)}
{P(\{ \mat{y}_i \}_{i=0}^M | \mat{x})}
=
\frac{P(\{ \mat{y}_i \}_{i=0}^M | \mat{x}, z = 0) P(z = 0)}
{\sum_{j=0}^M P(\{ \mat{y}_i \}_{i=0}^M | \mat{x}, z = j) P(z = j)}.
\end{equation*}

Recall that all outcomes of $z$ are equally probable,
$P(z=0) = P(z=1) = ... = P(z=M)$,
\begin{equation*}
P(z = 0 | \mat{x}, \{ \mat{y}_i \}_{i=0}^M)
=
\frac{P(\{ \mat{y}_i \}_{i=0}^M | \mat{x}, z = 0)}
{\sum_{j=0}^M P(\{ \mat{y}_i \}_{i=0}^M | \mat{x}, z = j)}.
\end{equation*}

Since $z$ is the index corresponding to the observation sampled from $\pi^{\ast}$,
\begin{equation*}
P(\{ \mat{y}_i \}_{i=0}^M | \mat{x}, z = i)
=
\pi^{\ast}(\mat{y}_i | \mat{x})
\prod_{j \neq i}
\mu(\mat{y}_j | \mat{x}).
\end{equation*}

Therefore,
\begin{align*}
P(z = 0 | \mat{x}, \{ \mat{y}_i \}_{i=0}^M)
& =\frac{P(\{ \mat{y}_i \}_{i=0}^M | \mat{x}, z = 0)}
{\sum_{j=0}^M P(\{ \mat{y}_i \}_{i=0}^M | \mat{x}, z = j)},
\\
& =
\frac{\pi^{\ast}(\mat{y}_0 | \mat{x}) \prod_{i=1}^M
\mu(\mat{y}_i | \mat{x})}
{\sum_{i=0}^M \pi^{\ast}(\mat{y}_i | \mat{x}) \prod_{j \neq i} \mu(\mat{y}_j | \mat{x})},
\\
& =
\frac{\pi^{\ast}(\mat{y}_0 | \mat{x}) / \mu(\mat{y}_0 | \mat{x}) \prod_{i=0}^M \mu(\mat{y}_i | \mat{x})}
{\sum_{i=0}^M \pi^{\ast}(\mat{y}_i | \mat{x}) / \mu(\mat{y}_i | \mat{x}) \prod_{j=0}^M \mu(\mat{y}_j | \mat{x})},
\\
& =
\frac{\pi^{\ast}(\mat{y}_0 | \mat{x}) / \mu(\mat{y}_0 | \mat{x})}
{\sum_{i=0}^M \pi^{\ast}(\mat{y}_i | \mat{x}) / \mu(\mat{y}_i | \mat{x})}.
\end{align*}

We use the probability model defined in Eq.~\eqref{eq: mle parameterized policy} to estimate $\pi^{\ast}$,
\begin{equation*}
\pi^{\ast}(\mat{y} | \mat{x})
\approx
p_{\boldsymbol\theta}(\mat{y} | \mat{x})
: =
\frac{1} {Z_{\boldsymbol\theta}(\mat{x})}
\mu(\mat{y} | \mat{x}) 
\exp\Big(
\beta
r_{\boldsymbol\theta}(\mat{x}, \mat{y})
\Big),
\end{equation*}
this leads to the following
\begin{equation*}
P(z = 0 | \mat{x}, \{ \mat{y}_i \}_{i=0}^M)
=
\frac{p_{\boldsymbol\theta}(\mat{y}_0 | \mat{x}) / \mu(\mat{y}_0 | \mat{x})}
{\sum_{i=0}^M p_{\boldsymbol\theta}(\mat{y}_i | \mat{x}) / \mu(\mat{y}_i | \mat{x})}
= 
\frac{\exp\Big(
\beta
r_{\boldsymbol\theta}(\mat{x}, \mat{y}_0)
\Big)}
{\sum_{i=0}^M \exp\Big(
\beta
r_{\boldsymbol\theta}(\mat{x}, \mat{y}_i)
\Big)}.
\end{equation*}

Finally,
\begin{align*}
\min \limits_{\boldsymbol\theta}
- \log P(z = 0 | \mat{x}, \{ \mat{y}_i \}_{i=0}^M)
& =
\min \limits_{\boldsymbol\theta}
- \log
\Big(
\frac{\exp\Big(
\beta
r_{\boldsymbol\theta}(\mat{x}, \mat{y}_0)
\Big)}
{\sum_{i=0}^M \exp\Big(
\beta
r_{\boldsymbol\theta}(\mat{x}, \mat{y}_i)
\Big)}
\Big)
\\
& =
\min \limits_{\boldsymbol\theta}
-
\Big(
\beta
r_{\boldsymbol\theta}(\mat{x}, \mat{y}_0)
\Big)
+
\log
{\sum_{i=0}^M \exp
\Big(
\beta
r_{\boldsymbol\theta}(\mat{x}, \mat{y}_i)
\Big)}.
\end{align*}
\end{proof}

\subsection{Proof of Lemma \ref{prop: cd generalizes rnce}}
\label{sec: cd generalizes rnce}
\cdgeneralizernce*

\begin{proof}
Let $\mat{y}_0 \sim \pi^{\ast}(\mat{y} | \mat{x})$
and
$\mat{y}_i \sim \mu(\mat{y} | \mat{x})$ for $i \in [1,...,M]$.
Recall the intractable gradient term in Eq.~\eqref{eq: gradient of mle objective function},
\begin{align}
\label{eq: cd estimation gradient of log norm constant}
\mathbb{E}_{p_{\boldsymbol\theta}(\mat{y} | \mat{x})}
\Big[
\nabla_{\boldsymbol\theta} 
\beta
r_{\boldsymbol\theta}(\mat{x}, \mat{y})
\Big]
& =
\mathbb{E}_{\mu(\mat{y} | \mat{x})} 
\Big[
\frac{p_{\boldsymbol\theta}(\mat{y} | \mat{x})} {\mu(\mat{y} | \mat{x})} 
\nabla_{\boldsymbol\theta} 
\beta
r_{\boldsymbol\theta}(\mat{x}, \mat{y})
\Big],
\\ \nonumber
& =
\mathbb{E}_{\mu(\mat{y} | \mat{x})} 
\Big[
\frac{
\exp 
\big(
\beta
r_{\boldsymbol\theta}(\mat{x}, \mat{y})
\big)
}
{Z_{\boldsymbol\theta}(\mat{x})}
\nabla_{\boldsymbol\theta} 
\beta
r_{\boldsymbol\theta}(\mat{x}, \mat{y})
\Big],
\\ \nonumber
& \approx
\sum_{i=0}^M
\frac{\exp 
\big(
\beta
r_{\boldsymbol\theta}(\mat{x}, \mat{y}_i)
\big)}
{\sum_{j=0}^M \exp 
\big(
\beta
r_{\boldsymbol\theta}(\mat{x}, \mat{y}_j)
\big)}
\nabla_{\boldsymbol\theta}
\beta
r_{\boldsymbol\theta}(\mat{x}, \mat{y}_i)
,
\\ \nonumber
& =
\nabla_{\boldsymbol\theta}
\log
{\sum_{i=0}^M \exp
\Big(
\beta
r_{\boldsymbol\theta}(\mat{x}, \mat{y}_i)
\Big)}.
\end{align}
Threfore,
CD applies importance sampling with a proposal distribution $\mu$ to estimate the gradient of the log-normalization constant.
\end{proof}

\subsection{Proof of Lemma \ref{prop: dpo gradient}}
\label{sec: dpo gradient derivation}
\dpogradient*
\begin{proof}
Recall the sampling-based solution of the NLL estimation objective function in Eq.~\eqref{eq: ranking noise contrastive estimation},
let $M=1$,
\begin{align*}
\nabla_{\boldsymbol\theta} \mathcal{L}^{Sample}(\boldsymbol\theta, \mat{x}, \mat{y}_0)
& =
\nabla_{\boldsymbol\theta}
\Big[
-
\beta
r_{\boldsymbol\theta}(\mat{x}, \mat{y}_0)
+
\log
{\sum_{i=0}^1 \exp
\Big(
\beta
r_{\boldsymbol\theta}(\mat{x}, \mat{y}_i)
\Big)}
\Big],
\\ \nonumber
& =
\nabla_{\boldsymbol\theta}
\Big[
-
\log
\frac{
\exp
\Big(
\beta
r_{\boldsymbol\theta}(\mat{x}, \mat{y}_0)
\Big)
}
{
\exp
\Big(
\beta
r_{\boldsymbol\theta}(\mat{x}, \mat{y}_0)
\Big)
+
\exp
\Big(
\beta
r_{\boldsymbol\theta}(\mat{x}, \mat{y}_1)
\Big)
}
\Big],
\\ \nonumber
& =
\nabla_{\boldsymbol\theta}
\Big[
-
\log \sigma
\Big(
\beta
r_{\boldsymbol\theta}(\mat{x}, \mat{y}_0)
-
\beta
r_{\boldsymbol\theta}(\mat{x}, \mat{y}_1)
\Big)
\Big],
\\ \nonumber
& =
- \Big(1 - \sigma
\Big(
\beta
r_{\boldsymbol\theta}(\mat{x}, \mat{y}_0)
-
\beta
r_{\boldsymbol\theta}(\mat{x}, \mat{y}_1)
\Big)
\Big)
\nabla_{\boldsymbol\theta}
\Big(
\beta
r_{\boldsymbol\theta}(\mat{x}, \mat{y}_0)
-
\beta
r_{\boldsymbol\theta}(\mat{x}, \mat{y}_1)
\Big),
\\ \nonumber
& =
- \beta \sigma
\Big(
\beta
r_{\boldsymbol\theta}(\mat{x}, \mat{y}_1)
-
\beta
r_{\boldsymbol\theta}(\mat{x}, \mat{y}_0)
\Big)
\nabla_{\boldsymbol\theta}
\Big(
r_{\boldsymbol\theta}(\mat{x}, \mat{y}_0)
-
r_{\boldsymbol\theta}(\mat{x}, \mat{y}_1)
\Big).
\end{align*}
The last equality uses $1 - \sigma(x) = \sigma(-x)$.
\end{proof}

\subsection{Proof of Proposition \ref{prop: unbiased mle estimator}}
\label{sec: unbiased mle estimator}

\begin{lemma}
\label{lemma: appendix probability simplification}
Let $z$ denote the index corresponding to the observation sampled from the probability model $p_{\boldsymbol\theta}$,
\begin{equation*}
\begin{cases} 
    \mat{y}_z \sim p_{\boldsymbol\theta}(\mat{y} | \mat{x}), \\
    \mat{y}_j \sim \mu(\mat{y} | \mat{x}),
    \;\; {\rm for} \; j \neq z, j = 0,1,...,M.
\end{cases}
\end{equation*}
Then the marginal distribution
\begin{equation*}
P(\{\mat{y}_i\}_{i=0}^M)
=
\frac{\hat{Z}_{\boldsymbol\theta}(\mat{x})} {Z_{\boldsymbol\theta}(\mat{x})}
\prod_{j=0}^M \mu(\mat{y}_j | \mat{x}).
\end{equation*}
\end{lemma}

\begin{proof}
Without loss of generality, we consider $z=0$.
Let 
\begin{equation*}
P(z=0, \{\mat{y}_i\}_{i=0}^M) = p_{\boldsymbol\theta}(\mat{y}_0 | \mat{x}) 
\prod_{i=1}^M \mu(\mat{y}_i | \mat{x}).
\end{equation*}
Then,
assume that all outcomes of $z$ are equally probable a priori,
e.g.,
$P(z=i) = 1 / (M + 1)$, for $i = 0, 1,...,M$,
\begin{equation*}
P(z=0, \{\mat{y}_i\}_{i=0}^M)
=
P(z = 0) 
\cdot
P(\mat{y}_0 | \mat{x}, z = 0)
\cdot
\prod_{i=1}^M P(\mat{y}_i | \mat{x}, z = 0)
=
\frac{1}{M + 1}
\cdot
p_{\boldsymbol\theta}(\mat{y}_0 | \mat{x})
\cdot
\prod_{i=1}^M \mu(\mat{y}_i | \mat{x}).
\end{equation*}

The marginal distribution $P(\{\mat{y}_i\}_{i=0}^M)$ can be computed by marginalizing $P(z, \{\mat{y}_i\}_{i=0}^M)$ over the index $z$.
\begin{align*}
P(\{\mat{y}_i\}_{i=0}^M)
& =
\sum_{i=0}^M
P(z=i, \{\mat{y}_i\}_{i=0}^M)
\\ \nonumber
& =
\sum_{i=0}^M
\frac{1}{M + 1}
\cdot
p_{\boldsymbol\theta}(\mat{y}_i | \mat{x})
\cdot
\prod_{j \neq i} \mu(\mat{y}_j | \mat{x}),
\\ \nonumber
& =
\sum_{i=0}^M
\frac{1}{M + 1}
\cdot
\frac{\mu(\mat{y}_i | \mat{x}) 
\exp
\big(\beta
r_{\boldsymbol\theta}(\mat{x}, \mat{y}_i)
\big) } {Z_{\boldsymbol\theta}(\mat{x})}
\frac{\mu(\mat{y}_i | \mat{x})} {\mu(\mat{y}_i | \mat{x})}
\cdot
\prod_{j \neq i} \mu(\mat{y}_j | \mat{x}),
\\ \nonumber
& =
\frac{1} {Z_{\boldsymbol\theta}(\mat{x})(M + 1)}
\cdot
\Big(
\prod_{j=0}^M \mu(\mat{y}_j | \mat{x})
\Big)
\Big(
\sum_{i=0}^M
\exp\big(
\beta
r_{\boldsymbol\theta}(\mat{x}, \mat{y}_i)
\big)
\Big),
\\ \nonumber
& =
\frac{\hat{Z}_{\boldsymbol\theta}(\mat{x})} {Z_{\boldsymbol\theta}(\mat{x})}
\prod_{j=0}^M \mu(\mat{y}_j | \mat{x}),
\end{align*}

where $\hat{Z}_{\boldsymbol\theta}(\mat{x}) = \frac{1} {M + 1} \sum_{i=0}^M \exp
\big(
\beta
r_{\boldsymbol\theta}(\mat{x}, \mat{y}_i)
\big)
$.
\end{proof}

\UnbiasedMLE*

\begin{proof}

We prove that $\nabla_{\boldsymbol\theta}\log\hat{Z}_{\boldsymbol\theta}(\mat{x})$ is an unbiased estimator when $\mat{y}_0$ is sampled from the probability model $p_{\boldsymbol\theta}$.
\begin{align*}
\mathbb{E}_{p_{\boldsymbol\theta}(\mat{y}_0 | \mat{x}) 
\mu(\{\mat{y}_i\}_{i=1}^M | \mat{x})
}
\Big[
\nabla_{\boldsymbol\theta} \log \hat{Z}_{\boldsymbol\theta}(\mat{x})
\Big]
& =
\mathbb{E}_{p_{\boldsymbol\theta}(\mat{y}_0 | \mat{x}) 
\mu(\{\mat{y}_i\}_{i=1}^M | \mat{x})
}
\Big[
\frac{1} {\hat{Z}(\mat{x})}
\nabla_{\boldsymbol\theta} \hat{Z}_{\boldsymbol\theta}(\mat{x})
\Big],
\\ \nonumber
& =
\frac{1} {M + 1}
\sum_{i=0}^M
\mathbb{E}_{p_{\boldsymbol\theta}(\mat{y}_0 | \mat{x}) 
\mu(\{\mat{y}_i\}_{i=1}^M | \mat{x})
}
\Big[
\frac{\nabla_{\boldsymbol\theta}\exp\big(
\beta
r_{\boldsymbol\theta}(\mat{x}, \mat{y}_i)
\big)
}
{\hat{Z}_{\boldsymbol\theta}(\mat{x})}
\Big],
\\ \nonumber
& =
\frac{1} {M + 1}
\sum_{i=0}^M
\mathbb{E}_{
P(z, \{\mat{y}_i \}_{i=0}^M)
}
\Big[
\frac{\nabla_{\boldsymbol\theta}\exp
\big( 
\beta
r_{\boldsymbol\theta}(\mat{x}, \mat{y}_i)
\big)
} 
{\hat{Z}_{\boldsymbol\theta}(\mat{x})}
\Big],
\end{align*}
where 
$p_{\boldsymbol\theta}(\mat{y}_0 | \mat{x}) 
\mu(\{\mat{y}_i\}_{i=1}^M | \mat{x})
=
P(z, \{\mat{y}_i \}_{i=0}^M)
$ by definition.

Since the integrand 
$
\frac{\nabla_{\boldsymbol\theta}\exp\big(
\beta
r_{\boldsymbol\theta}(\mat{x}, \mat{y}_i)
\big)
} 
{\hat{Z}_{\boldsymbol\theta}(\mat{x})}
$ does not depend on the index $z$,
\begin{align}
\label{eq: appendix derivation unbiased gradient 2}
\mathbb{E}_{p_{\boldsymbol\theta}(\mat{y}_0 | \mat{x}) 
\mu(\{\mat{y}_i\}_{i=1}^M | \mat{x})
}
\Big[
\nabla_{\boldsymbol\theta} \log \hat{Z}_{\boldsymbol\theta}(\mat{x})
\Big]
=
\frac{1} {M + 1}
\sum_{i=0}^M
\mathbb{E}_{
P(\{\mat{y}_i \}_{i=0}^M)
}
\Big[
\frac{\nabla_{\boldsymbol\theta}\exp\big(
\beta
r_{\boldsymbol\theta}(\mat{x}, \mat{y}_i)
\big)
} 
{\hat{Z}_{\boldsymbol\theta}(\mat{x})}
\Big].
\end{align}

Recall that 
$P(\{\mat{y}_i\}_{i=0}^M)
=
\frac{\hat{Z}_{\boldsymbol\theta}(\mat{x})} {Z_{\boldsymbol\theta}(\mat{x})}
\prod_{j=0}^M \mu(\mat{y}_j | \mat{x})$ from Lemma \ref{lemma: appendix probability simplification},

\begin{align*}
\mathbb{E}_{P(\{\mat{y}_i \}_{i=0}^M)
}
\Big[
\frac{\nabla_{\boldsymbol\theta}\exp\big( 
\beta
r_{\boldsymbol\theta}(\mat{x}, \mat{y}_i)
\big)} 
{\hat{Z}_{\boldsymbol\theta}(\mat{x})
}
\Big]
& =
\int
\frac{\nabla_{\boldsymbol\theta}\exp\big(
\beta
\log 
r_{\boldsymbol\theta}(\mat{x}, \mat{y}_i)
\big)}
{\hat{Z}_{\boldsymbol\theta}(\mat{x})}
\cdot
\frac{\hat{Z}_{\boldsymbol\theta}(\mat{x})} {Z_{\boldsymbol\theta}(\mat{x})}
\cdot
\prod_{i=0}^M \mu(\mat{y}_i | \mat{x})
\cdot
d \{ \mat{y}_i \}_{i=0}^M,
\\
& =
\frac{1}{Z_{\boldsymbol\theta}(\mat{x})}
\int
\nabla_{\boldsymbol\theta}\exp\big( 
\beta
r_{\boldsymbol\theta}(\mat{x}, \mat{y}_i)
\big)
\cdot
\mu(\mat{y}_i | \mat{x})
\cdot
d\mat{y}_i
\cdot
\prod_{j \neq i}
\mu(\mat{y}_i | \mat{x}) 
\cdot
d\{\mat{y}_j\}_{j\neq i},
\end{align*}
where 
$
\prod_{j \neq i}
\mu(\mat{y}_i | \mat{x}) d\{\mat{y}_j\}_{j\neq i} = 1
$.
Then,
\begin{align*}
\mathbb{E}_{P(\{\mat{y}_i \}_{i=0}^M)
}
\Big[
\frac{\nabla_{\boldsymbol\theta}\exp\big( 
\beta
r_{\boldsymbol\theta}(\mat{x}, \mat{y}_i)
\big)} 
{\hat{Z}_{\boldsymbol\theta}(\mat{x})
}
\Big]
& =
\frac{1}{Z_{\boldsymbol\theta}(\mat{x})}
\int
\nabla_{\boldsymbol\theta}\exp\big( 
\beta
r_{\boldsymbol\theta}(\mat{x}, \mat{y}_i)
\big)
\cdot
\mu(\mat{y}_i | \mat{x})
\cdot
d\mat{y}_i,
\\
& =
\frac{1}{Z_{\boldsymbol\theta}(\mat{x})}
\int
\exp\big(
\beta
r_{\boldsymbol\theta}(\mat{x}, \mat{y})
\big)
\cdot
\nabla_{\boldsymbol\theta}
\log
\exp\big(
\beta
r_{\boldsymbol\theta}(\mat{x}, \mat{y})
\big)
\cdot
\mu(\mat{y} | \mat{x}) 
\cdot
d\mat{y}.
\end{align*}

Recall the expression of the probability model,
$
\exp\big(
\beta
r_{\boldsymbol\theta}(\mat{x}, \mat{y}_i)
\big)
=
\frac{Z_{\boldsymbol\theta}(\mat{x}) p_{\boldsymbol\theta}(\mat{y} | \mat{x})} {\mu(\mat{y} | \mat{x})}
$.
Then,
\begin{align*}
\label{eq: appendix derivation unbiased gradient 3}
\mathbb{E}_{P(\{\mat{y}_i \}_{i=0}^M)
}
\Big[
\frac{\nabla_{\boldsymbol\theta}\exp\big( 
\beta
r_{\boldsymbol\theta}(\mat{x}, \mat{y}_i)
\big)} 
{\hat{Z}_{\boldsymbol\theta}(\mat{x})
}
\Big]
& =
\frac{1}{Z_{\boldsymbol\theta}(\mat{x})}
\int
\frac{Z_{\boldsymbol\theta}(\mat{x}) p_{\boldsymbol\theta}(\mat{y} | \mat{x})} {\mu(\mat{y} | \mat{x})}
\cdot
\nabla_{\boldsymbol\theta}
\log 
\frac{Z_{\boldsymbol\theta}(\mat{x}) p_{\boldsymbol\theta}(\mat{y} | \mat{x})} {\mu(\mat{y} | \mat{x})}
\cdot
\mu(\mat{y} | \mat{x}) 
\cdot
d\mat{y},
\\ \nonumber
& =
\int
p_{\boldsymbol\theta}(\mat{y} | \mat{x})
[
\nabla_{\boldsymbol\theta} \log Z_{\boldsymbol\theta}(\mat{x})
+
\nabla_{\boldsymbol\theta} \log p_{\boldsymbol\theta}(\mat{y} | \mat{x}) 
-
\nabla_{\boldsymbol\theta} \log \mu(\mat{y} | \mat{x})
]
d\mat{y},
\\ \nonumber
& =
\int p_{\boldsymbol\theta}(\mat{y} | \mat{x})
\nabla_{\boldsymbol\theta} \log Z_{\boldsymbol\theta}(\mat{x})
d \mat{y}
+
\int p_{\boldsymbol\theta}(\mat{y} | \mat{x})
\nabla_{\boldsymbol\theta} \log p_{\boldsymbol\theta}(\mat{y} | \mat{x})
d \mat{y},
\\ \nonumber
& =
\nabla_{\boldsymbol\theta} \log Z_{\boldsymbol\theta}(\mat{x})
\int
p_{\boldsymbol\theta}(\mat{y} | \mat{x}) d\mat{y}
+
\int
\nabla_{\boldsymbol\theta}
p_{\boldsymbol\theta}(\mat{y} | \mat{x}) d\mat{y},
\\ \nonumber
& =
\nabla_{\boldsymbol\theta} \log Z_{\boldsymbol\theta}(\mat{x})
+
\nabla_{\boldsymbol\theta}
\int
p_{\boldsymbol\theta}(\mat{y} | \mat{x}) d\mat{y},
\\ \nonumber
& =
\nabla_{\boldsymbol\theta} \log Z_{\boldsymbol\theta}(\mat{x}).
\end{align*}

Recall in Eq.~\eqref{eq: appendix derivation unbiased gradient 2},
\begin{align*}
\mathbb{E}_{p_{\boldsymbol\theta}(\mat{y}_0 | \mat{x}) 
\mu(\{\mat{y}_i\}_{i=1}^M | \mat{x})
}
\Big[
\nabla_{\boldsymbol\theta} \log \hat{Z}_{\boldsymbol\theta}(\mat{x})
\Big]
& = 
\frac{1} {M + 1}
\sum_{i=0}^M
\mathbb{E}_{
P(\{\mat{y}_i \}_{i=0}^M)
}
\Big[
\frac{\nabla_{\boldsymbol\theta}\exp\big(
\beta
r_{\boldsymbol\theta}(\mat{x}, \mat{y}_i)
\big)
} 
{\hat{Z}_{\boldsymbol\theta}(\mat{x})}
\Big],
\\ \nonumber
& =
\frac{1}{M + 1}
\sum_{i=0}^M
\nabla_{\boldsymbol\theta} \log Z_{\boldsymbol\theta}(\mat{x}),
\\ \nonumber
& =
\nabla_{\boldsymbol\theta} \log Z_{\boldsymbol\theta}(\mat{x}).
\end{align*}
\end{proof}

\newpage
\section{Extensive Experimental Results}
\label{appendix: extensive experimental results}

\subsection{Experimental Setup}
\label{appendix: experimental setup}
We provide details about experimental setting here.

\paragraph{Models and datasets.}
Our main experiments are conducted under three settings.

\begin{itemize}
\item \textbf{Base setup:}
This setting considers the 
\href{https://huggingface.co/allenai/Llama-3.1-Tulu-3-8B-SFT}{Llama-3.1-8B-SFT} model, which has been fine-tuned using supervised next-word prediction on the TÜLU 3 SFT Mix dataset \citep{lambert2024t},
and {\href{https://huggingface.co/HuggingFaceH4/mistral-7b-sft-beta}{Mistral-7B-SFT}}.

The \href{https://huggingface.co/allenai/Llama-3.1-Tulu-3-8B-SFT}{Llama-3.1-8B} model is constructed by fine-tuning the \href{https://huggingface.co/meta-llama/Llama-3.1-8B}{Llama-3.1-8B-base} on the TÜLU 3 SFT Mix dataset. 
The TÜLU 3 SFT Mix dataset spans various domains including general instruction following, knowledge recall, mathematical reasoning, coding, safety, non-compliance, and multilingual tasks, with domain mixing ratios determined by thorough experimental analyses and contains approximately $23$M prompt-response pairs. 
We employ the publicly available model checkpoint of the \href{https://huggingface.co/allenai/Llama-3.1-Tulu-3-8B-SFT}{Llama-3.1-8B} for further fine-tuning on the \href{https://huggingface.co/datasets/berkeley-nest/Nectar}{Nectar} dataset \citep{zhu2023starling}, which includes $7$ ranked completions per input prompt generated by various LLMs, providing a diverse and high-quality set of candidate completions. 
The Nectar dataset is modified by removing the rank-$2$ completion, leaving each prompt with $5$ ranked completions.
For each prompt, the rank-$1$ completion is considered as the preferred completion, and a dispreferred completion is randomly selected from the remaining $5$ candidates.
\item \textbf{Instruct setup:}
This setup considers the off-the-shelf instruction-tuned
\href{https://huggingface.co/meta-llama/Llama-3.1-8B-Instruct}{Llama-3.1-8B-Instruct} model to initialize the target policy $\pi_{\boldsymbol\theta}$.
This model has undergone extensive instruction-tuning processes, making it more expressive compared to the initialization model in the base setup.

We utilize prompts from the \href{https://huggingface.co/datasets/HuggingFaceH4/ultrafeedback_binarized}{UltraFeedback} dataset \citep{cui2023ultrafeedback} to regenerate both chosen and rejected completions employing the \href{https://huggingface.co/meta-llama/Llama-3.1-8B-Instruct}{Llama-3.1-8B-Instruct} model. 
This approach aligns the instruct setup more closely with an on-policy framework \citep{tang2024understanding}.
Specifically, for each prompt, we generate two completions at a temperature of $0.6$, two at $0.8$, and two at $1.0$, thereby introducing diversity within the candidate completions. 
Subsequently, we implement the iterative pairwise ranking method \citep{chen2024towards} using the \href{https://huggingface.co/meta-llama/Llama-3.1-70B-Instruct}{Llama-3.1-70B-Instruct} \citep{dubey2024llama} to determine the most preferred completion and randomly select a dispreferred completion from the remaining candidates.
The iterative pairwise ranking algorithm \citep{chen2024towards} relies on two assumptions to identify the winner:
\begin{enumerate}[noitemsep]
    \item
    \textbf{Transitive:}
    $y^{(i, a)} \succ y^{(i, b)}$ and $y^{(i, b)} \succ y^{(i, c)}$ leads to $y^{(i, a)} \succ y^{(i, c)}$ almost surely, where $a, b, c \in \{1, 2, \dots, M\}$.
    \item 
    \textbf{Symmetry:}
    The ordering of two completions does not affect the comparison result $W$,
    $W(x^i, y^{(i, a)}, y^{(i, b)}) = W(x^i, y^{(i, b)}, y^{(i, a)})$.
\end{enumerate}
Given these assumptions, identifying the most preferred completion from $L$ candidates can be accomplished from $(L-1)$ comparisons.
Specifically, the algorithm initiates by comparing the first pair of completions, followed by comparing their winner with the next candidate. 
This iterative process continues until an overall winner is determined.
\item \textbf{Batched online setup:}
This setting is the middle of the offline and purely online setups \citep{schulman2017proximal, lambert2024t}, striking a balance between deployment efficiency and adaptability.
The number of total training steps is calculated as the number of total data divided by the effective batch size (the effective batch size is chosen as $128$ across all experiments).
The training steps are then divided equally into three segments, and we use the model checkpoint from the start of each segment to regenerate the preference data. 
For example, with a total of $450$ training steps, we initiate with the \href{https://huggingface.co/meta-llama/Llama-3.1-8B-Instruct}{Llama-3.1-8B-Instruct} model to generate preference data for the first $150$ steps. 
At the $\rm 150^{th}$ step, we utilize the current model checkpoint to generate data for the next $150$ steps, continuing this sequence. 
The preference data generation adheres to the Instruct setting. 
This method proves more efficient than a purely online approach \citep{schulman2017proximal,qi2024online}, as starting the inference kernel in an online environment often incurs significant computational costs \citep{kwon2023efficient}.
\end{itemize}

\paragraph{MC-PO implementation details.}
Recall the MC kernel defined in Algorithm \ref{alg: contrastive divergence kernel},
this kernel selects a output based on samples from a proposal distribution $\mu$.
In the instruct setup,
the proposal distribution is considered as the reference policy $\pi_{\rm ref}$,
that is the \href{https://huggingface.co/meta-llama/Llama-3.1-8B-Instruct}{Llama-3.1-8B-Instruct} model.
In the base setup,
we consider the proposal distribution as a list of LLMs that are used to generate all completions in the \href{https://huggingface.co/datasets/berkeley-nest/Nectar}{Nectar} dataset.
These LLMs include GPT-4, GPT-3.5-turbo, GPT-3.5-turbo-instruct, LLama-2-7B-chat, and Mistral-7B-Instruct, alongside other existing datasets and models.

\paragraph{Training.}
All training jobs are done using full-parameter tuning. 
We fix the batch size as $1$ and gradient accumulation steps as $128$,
which results in an effective batch size of $128$.
We train all models with $2$ epochs.
Hyperparameter optimization is conducted using $7$ different learning rates.
All results are reported as the average performance of the final checkpoints across $3$ random seeds, along with the standard deviation,
which can effectively reduce numerical randomness \citep{miller2024adding}.
Each training job is done on a node of $8\cdot$A100 GPUs and multiple nodes are executed in parallel.

\begin{itemize}
\item \textbf{Justification on $2$ training epochs.}
\begin{table}[h!]
\centering
\begin{tabular}{l|ccc}
\hline
Model & \multicolumn{3}{|c}{\href{https://huggingface.co/allenai/Llama-3.1-Tulu-3-8B-SFT}{Llama-3.1-8B-Base} (Alpaca-Eval)} \\
& Epoch $1$ & Epoch $2$ & Epoch $3$ \\
MC-PO & 32.93($\pm$0.39) & 35.84($\pm$0.31) & 35.01($\pm$0.71) \\ 
\hline
Model & \multicolumn{3}{|c}{\href{https://huggingface.co/allenai/Llama-3.1-Tulu-3-8B-SFT}{Llama-3.1-8B-Base} (Arena)} \\
& Epoch $1$ & Epoch $2$ & Epoch $3$ \\
MC-PO & 61.70($\pm$0.29) & 63.77($\pm$0.81) & 63.83($\pm$0.75) \\ 
\hline
\end{tabular}
\caption{
Performance of preference-optimized models using MC-PO at each training epoch.
}
\label{table: mc-po on all epochs}
\end{table}
As shown in Table \ref{table: mc-po on all epochs},
the MC-PO training from epoch $1$ to epoch $2$ demonstrates substantial performance improvement.
Extending training to $3$ epochs does not yield additional improvements in performance.

\item 
\textbf{Details on hyperparameter optimization.}
We choose $7$ learning rates for all PO algorithms.
These include 
$8e-7$,
$1e-6$,
$2e-6$,
$5e-6$,
$8e-7$,
$1e-5$,
$2e-5$.
Since each experiment is repeated with three random seeds,
each reported number in the experiment section requires training $3 \times 7 = 21$ models.

\end{itemize}

\paragraph{Evaluation.}
We compute winrate with \href{https://huggingface.co/mistralai/Mistral-Large-Instruct-2407}{Mistral-Large-Instruct-2407  } as the model judge for all evaluations.
The input prompt for the LLM judge is shown as follows,

\begin{codebox}
You are a helpful assistant, that ranks models by the quality of their answers. \
Act as an impartial judge and evaluate the quality of the responses provided by two AI assistants to the user question displayed below. \
The length of the response generated by each assistant is not a criterion for evaluation. \
Your evaluation should consider correctness, helpfulness, completeness, and clarity of the responses. \
Remember not to allow the length of the responses to influence your evaluation. \
You will be given the question within <question> tags, \
assistant A's answer within <assistant a> tags. \
and assistant B's answer within <assistant b> tags. \
Your job is to evaluate whether assistant A's answer or assistant B's answer is better. \
Avoid any position biases and ensure that the order in which the responses are presented does not \
influence your decision. Be as objective as possible. \
After providing your explanation, output your final verdict within <verdict> tags strictly following this format: \
<verdict>A</verdict> if assistant A is better, <verdict>B</verdict> if assistant B is better, and <verdict>tie</verdict> for a tie.
You must provide your final verdict with the format <verdict>xxx</verdict> once in your response!!!

<question>
{question}
</question>

<assistant a>
{response a}
</assistant a>

<assistant b>
{response b}
</assistant b>
\end{codebox}

\subsection{Details on Baseline Preference Optimization Algorithms}
\label{appendix: baseline preference optimization algorithms}
All baseline algorithms are implemented in the {\href{https://huggingface.co/docs/trl/en/index}{TRL library},
and their objective functions are summarized in Table \ref{table: preference optimization baseline algorithms}.
Here we present the details about their hyper-parameter choices.
The hyper-parameter $\beta$ is chosen as $0.01$ in all PO algorithms (if it appears).
The hyper-parameter $\lambda$ for the supervised next-word prediction is set as $0.1$.
$\gamma$ in SimPO and CPO is fixed as $10$.

\begin{table*}[h!]
\centering
\begin{tabular}{l|cc}
\hline
\multicolumn{1}{c|}{Method}  & Objective Function \\ \hline
DPO & 
$
- \log \sigma
\Big(
\beta
\log
\frac{\pi_{\boldsymbol\theta}(\mat{y}_0 | \mat{x})} 
{\pi_{\rm ref}(\mat{y}_0 | \mat{x})}
-
\beta
\log
\frac{\pi_{\boldsymbol\theta}(\mat{y}_1 | \mat{x})} 
{\pi_{\rm ref}(\mat{y}_1 | \mat{x})}
\Big)
$
\\
RPO & 
$
- \log \sigma
\Big(
\beta
\log
\frac{\pi_{\boldsymbol\theta}(\mat{y}_0 | \mat{x})} 
{\pi_{\rm ref}(\mat{y}_0 | \mat{x})}
-
\beta
\log
\frac{\pi_{\boldsymbol\theta}(\mat{y}_1 | \mat{x})} 
{\pi_{\rm ref}(\mat{y}_1 | \mat{x})}
\Big)
-
\lambda
\log
\frac{\pi_{\boldsymbol\theta}(\mat{y}_0 | \mat{x})} 
{\pi_{\rm ref}(\mat{y}_0 | \mat{x})}
$
\\
EXO &
$
-
\sigma \Big(
\beta 
\cdot
{\rm logits}
\Big)
\log
\sigma \Big(
\beta 
\cdot
{\rm logits}
\Big)
+
\sigma \Big(
\beta 
\cdot
{\rm logits}
\Big)
\log
\sigma \Big(
-
\beta 
\cdot
{\rm logits}
\Big)
$
\\
SimPO &
$
- \log \sigma \Big(\frac{\beta} {|\mat{y}_0|} \log \pi_{\boldsymbol\theta} (\mat{y}_0 | \mat{x}) -  \frac{\beta} {|\mat{y}_1|} \log \pi_{\boldsymbol\theta} (\mat{y}_1 | \mat{x}) - \gamma \Big)
$ 
\\
CPO &
$
- \log \sigma \Big(
\frac{\beta} {|\mat{y}_0|} \log \pi_{\boldsymbol\theta} (\mat{y}_0 | \mat{x}) -  \frac{\beta} {|\mat{y}_1|} \log \pi_{\boldsymbol\theta} (\mat{y}_1 | \mat{x}) - \gamma \Big)
-
\lambda \cdot
\frac{\beta} {|\mat{y}_0|} \log \pi_{\boldsymbol\theta} (\mat{y}_0 | \mat{x})
$
\\
\hline
BCO & 
$
- \log \sigma
\Big(
\beta
\log
\frac{\pi_{\boldsymbol\theta}(\mat{y}_0 | \mat{x})} 
{\pi_{\rm ref}(\mat{y}_0 | \mat{x})} 
- \Delta_{\rm BCO}
\Big)
-
\log \sigma
\Big(
-
\beta
\log
\frac{\pi_{\boldsymbol\theta}(\mat{y}_1 | \mat{x})} 
{\pi_{\rm ref}(\mat{y}_1 | \mat{x})}
-
\Delta_{\rm BCO}
\Big)
$
\\
KTO & 
$
- \log \sigma
\Big(
\beta
\log
\frac{\pi_{\boldsymbol\theta}(\mat{y}_0 | \mat{x})} 
{\pi_{\rm ref}(\mat{y}_0 | \mat{x})} 
- \Delta_{\rm KTO}
\Big)
-
\log \sigma
\Big(
-
\beta
\log
\frac{\pi_{\boldsymbol\theta}(\mat{y}_1 | \mat{x})} 
{\pi_{\rm ref}(\mat{y}_1 | \mat{x})}
-
\Delta_{\rm KTO}
\Big)
$
\\
APO & 
$
- \log \sigma
\Big(
\beta
\log
\frac{\pi_{\boldsymbol\theta}(\mat{y}_0 | \mat{x})} 
{\pi_{\rm ref}(\mat{y}_0 | \mat{x})}
\Big)
+
\log \sigma
\Big(
\beta
\log
\frac{\pi_{\boldsymbol\theta}(\mat{y}_1 | \mat{x})} 
{\pi_{\rm ref}(\mat{y}_1 | \mat{x})}
\Big)
$
\\
SPPO & 
$
\Big(
\log
\frac{\pi_{\boldsymbol\theta}(\mat{y}_0 | \mat{x})} 
{\pi_{\rm ref}(\mat{y}_0 | \mat{x})}
-
\frac{0.5} {\beta}
\Big)^2
+
\Big(
\log
\frac{\pi_{\boldsymbol\theta}(\mat{y}_1 | \mat{x})} 
{\pi_{\rm ref}(\mat{y}_1 | \mat{x})}
+
\frac{0.5} {\beta}
\Big)^2
$
\\
NCA &
$
- \log \sigma
\Big(
\beta
\log
\frac{\pi_{\boldsymbol\theta}(\mat{y}_0 | \mat{x})} 
{\pi_{\rm ref}(\mat{y}_0 | \mat{x})}
\Big)
-
0.5 
\log \sigma
\Big(
-
\beta
\log
\frac{\pi_{\boldsymbol\theta}(\mat{y}_0 | \mat{x})} 
{\pi_{\rm ref}(\mat{y}_0 | \mat{x})}
\Big)
-
0.5 
\log \sigma
\Big(
-
\beta
\log
\frac{\pi_{\boldsymbol\theta}(\mat{y}_1 | \mat{x})} 
{\pi_{\rm ref}(\mat{y}_1 | \mat{x})}
\Big)
$
\\
\hline
MC-PO &
$
-
\beta
\log
\frac{\pi_{\boldsymbol\theta}(\mat{y}_0 | \mat{x})} 
{\pi_{\rm ref}(\mat{y}_0 | \mat{x})}
+
\log
{\sum_{i=0}^M \exp
\Big(
\beta
\log
\frac{\pi_{\boldsymbol\theta}(\mat{y}_i | \mat{x})} 
{\pi_{\rm ref}(\mat{y}_i | \mat{x})}
\Big)}.
$
\\ \hline
\end{tabular}
\caption{
Preference optimization algorithms and their objective function implementations.
In RPO and CPO: $\lambda$ is a hyper-parameter controlling the supervised next-word prediction regularization.
In EXO: 
$\rm{logits} := \log
\frac{\pi_{\boldsymbol\theta}(\mat{y}_0 | \mat{x})} 
{\pi_{\rm ref}(\mat{y}_0 | \mat{x})}$.
In SimPO: $\gamma$ is a hyper-parameter.
In BCO and KTO: $\Delta_{\rm BCO}$ and $\Delta_{\rm KTO}$ are empirically computed.
}
\label{table: preference optimization baseline algorithms}
\end{table*}

\end{document}